%% file: neurips_2026.tex
\documentclass{article}

\PassOptionsToPackage{numbers,compress}{natbib}
\usepackage[preprint]{neurips_2026}
\usepackage{microtype}
\usepackage{graphicx}
\usepackage{subcaption}
\usepackage{booktabs} 
\usepackage{paralist}
\usepackage{algorithm}
\usepackage{algorithmic} 
\usepackage{float}
\usepackage{amsmath}
\usepackage{amssymb}
\usepackage{mathtools}
\usepackage{amsthm}
\usepackage{dblfloatfix}
\usepackage{wrapfig}
\usepackage{hyperref}
\usepackage[capitalize,noabbrev]{cleveref}
\usepackage{multirow}
\usepackage{makecell}
\usepackage{placeins}
\usepackage{xcolor}
\usepackage[most]{tcolorbox}

\usepackage{amsmath,amssymb}
\usepackage{tikz-cd}
\usepackage{multicol}
\usepackage{booktabs}
\usepackage{longtable}
\usepackage{placeins}

\newcommand{\rel}[1]{\textcolor{purple}{\scriptsize (#1\%)}}
\newcommand{\cellrel}[2]{#1\,\rel{#2}}

\newcommand{\CP}{{\rm CP}}
\newcommand{\CPr}{\CP_{r\text{-value}}}
\newcommand{\CPavg}{\CP_{\text{avg}}}

\theoremstyle{plain}
\newtheorem{theorem}{Theorem}[section]
\newtheorem{proposition}[theorem]{Proposition}

\theoremstyle{definition}

\theoremstyle{remark}

\tcbset{
  aibox/.style={
    width=\linewidth,   
    top=10pt,
    colback=white,
    colframe=black,
    colbacktitle=black,
    enhanced,
    center,
    attach boxed title to top left={yshift=-0.1in,xshift=0.15in},
    boxed title style={boxrule=0pt,colframe=white,},
  }
}

\newtcolorbox{AIbox}[2][]{aibox,title=#2,#1}

\usepackage[utf8]{inputenc} 
\usepackage[T1]{fontenc}    
\usepackage{hyperref}       
\usepackage{url}            
\usepackage{booktabs}       
\usepackage{amsfonts}       
\usepackage{nicefrac}       
\usepackage{microtype}      
\usepackage{xcolor}         

\title{Empirical Bayes Conformal Prediction for Vision and Language Models}

\author{%
  Jiapeng Zeng$^{1}$ \quad
  Yogesh Prabhu$^{2}$ \quad
  Zhanpeng Zeng$^{3}$ \quad
  Michael A.~Newton$^{1}$ \quad
  Vikas Singh$^{1}$ \\
  $^{1}$University of Wisconsin--Madison \\
  $^{2}$University of California San Diego \\
  $^{3}$Xiamen University \\
}

\begin{document}

\maketitle

\begin{abstract}

    Conformal prediction (CP) gives distribution-free coverage for modern vision and language models, but it is often forced to make a ranking decision from a single unstable nonconformity score. Standard CP uses one realization, while average-then-calibrate variants smooth multiple realizations into a point estimate. Both options discard the inconsistency that can help identify whether a candidate is indeed stable. A weak answer can enter the conformal set even if the evidence is not strong, simply because one posterior sample or prompt phrasing made it look strong. But variability can help distinguish a stable signal from noise-driven fluctuations. We describe an empirical Bayes conformal prediction framework that uses $r$-values to convert score variability into an uncertainty informed nonconformity score. The resulting $r$-value estimates how likely a candidate's latent score belongs to the top-ranked group after accounting for both its mean score and its uncertainty. It admits both a closed-form Normal-Normal empirical Bayes estimator and a nonparametric posterior-sampling estimator. Using the $r$-value as the nonconformity score preserves the target conformal coverage while provably reducing the inclusion of high variance false candidates under mild regularity conditions. Across image classification, CLIP-based VLM benchmarks, and LLMs, we show that $r$-value conformal prediction preserves target coverage while improving ranking stability and reducing set size when variability is informative, and reverting to CP-like behavior when variability vanishes.

\end{abstract}

\section{Introduction}
\label{sec:intro}

\input{AISTATS/sections/Introduction}

\section{Preliminaries}
\label{sec:conformal}
\input{AISTATS/sections/preliminary}

\section{What is an $r$-value?}
\input{AISTATS/sections/R-value}

\section{Conformal Coverage using $r$-value for Vision Models}

\input{AISTATS/sections/Exp_Image}

\section{Conformal Coverage using $r$-value for VLMs and LLMs}
\input{AISTATS/sections/Exp_LLM}

\section{Related Work}

\input{AISTATS/sections/Relative_Work}

\section{Conclusion}

\input{AISTATS/sections/conclusion}

\bibliography{AISTATS/refs}
\bibliographystyle{plainnat}

\appendix
\newpage
\input{AISTATS/sections/Appendix}

\end{document}

%% file: AISTATS/sections/Introduction.tex
Large language models (LLMs) and vision models such as Vision Transformers (ViTs) \citep{dosovitskiy-2021, zhai-2021, li-2025} are increasingly being used to inform decision making, e.g., in clinical deployments and finance where quantifying uncertainty and mitigating risk are essential. Bayesian neural networks (BNNs) \cite{jospin-2022}, MC-dropout \cite{gal-2015} and Deep Ensembles (DE) \cite{lakshminarayanan-2016} are important ideas in linking trustworthiness and model variability. But applying these methods directly to large pre-trained models can be  expensive \cite{abdar-2021}. 
Conformal prediction ($\CP$) \cite{Vovk2005} provides an alternative framework where instead of estimating uncertainty inside the model, we wrap the model's scores with a calibrated prediction set. Rather than attaching uncertainty to a single prediction, CP returns a set guaranteed to contain the true output with user-specified probability under exchangeability \cite{angelopoulos-2023}. These properties have made CP a promising tool for both vision and language models \cite{yadkori-2024, quach-2023}.

{\bf The ranking problem in Conformal Prediction.} Despite its distribution-free coverage guarantee, $\CP$ is not designed to directly use epistemic uncertainty from model parameters or training randomness.
Standard CP relies on a single model output/score. The guarantee controls whether the true label is included, but the usefulness of the set relies on how well this score orders true candidates ahead of false ones. If several candidate labels receive similar scores, small fluctuations can change their order and therefore change the conformal set. This ranking instability becomes more pronounced under model or data heterogeneity, where score variability is no longer negligible.
Existing results \cite{plassier-2024} show that $\CP$ can struggle under statistical heterogeneity, resulting in unreliable coverage. 
Fig.~\ref{fig:intro} shows an example. For the same input image, different model instances sampled from the same posterior distribution of the model can produce different conformal prediction sets. The takeaway is that score variability should not simply be ignored. Instead, it can provide useful information for constructing more stable and efficient conformal rankings.

\begin{wrapfigure}[14]{r}{0.40\textwidth}
    \vspace{-10pt}
    \centering
    \includegraphics[width=\linewidth]{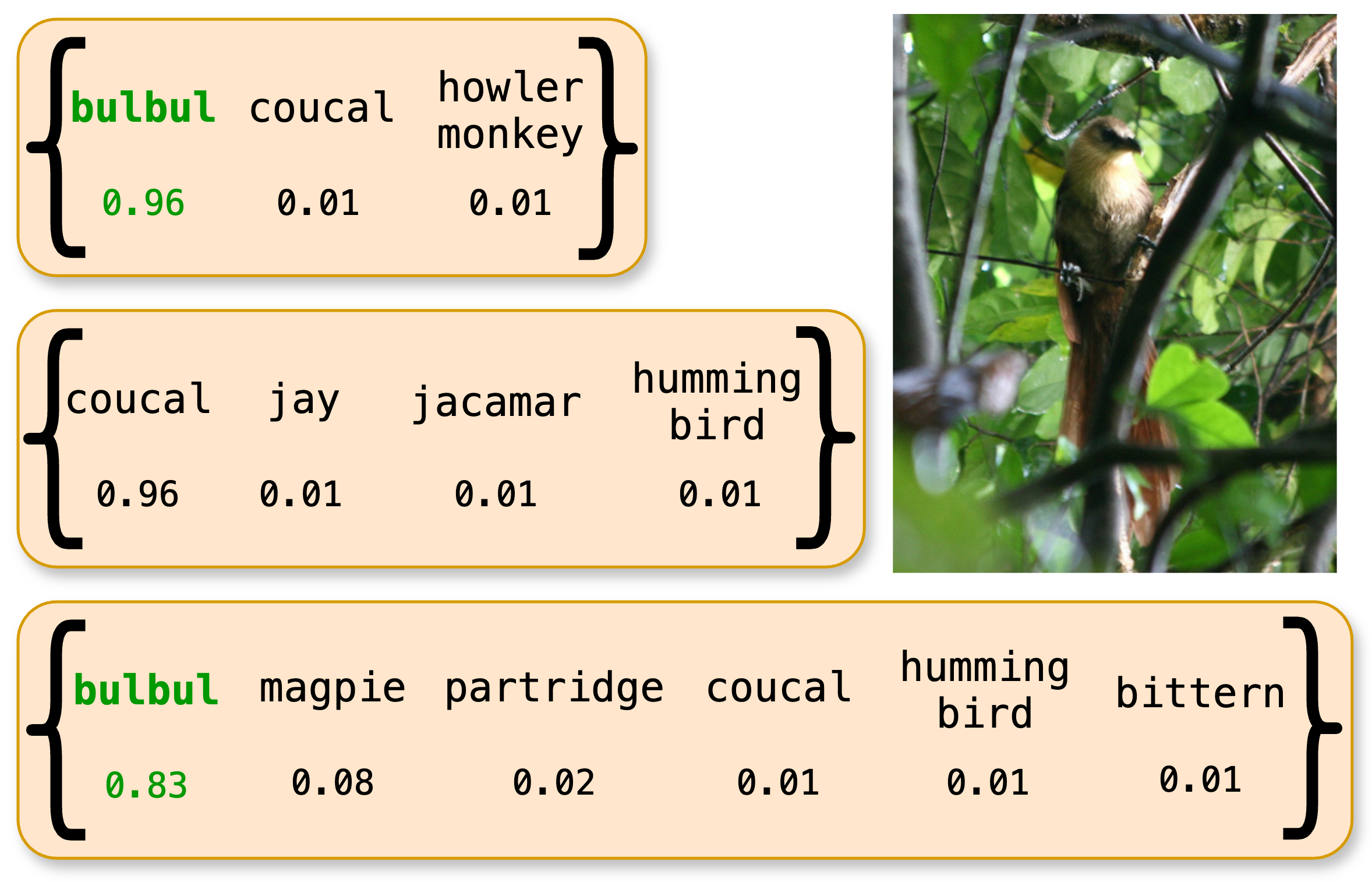}
    \vspace{-8pt}
    \caption{Posterior sampled models yield different CP sets for the same image, revealing epistemic instability.}
    \label{fig:intro}
    \vspace{-2pt}
\end{wrapfigure}

This instability is not just due to posterior model sampling. In vision models, it may arise from posterior uncertainty over model parameters; in VLMs and LLMs, it may arise from prompt paraphrasing or model based evaluation of candidate responses. A high score can mean two very different things: stable evidence for a candidate or a fluctuation of an otherwise weak one (which is when high-variance candidates enter the conformal set). $\CP$ sees only a single noisy realization, while $\CPavg$ averages realizations before calibration and smoothing this variability. Instead of eliminating variability, our goal is to check what it tells us about the reliability of a candidate's rank.

The {\bf contribution} of this work is to utilize what is considered a weakness of large models, namely their inconsistency \cite{chen-2023,mitchell-2022}, as a signal for conformal efficiency. 
We introduce an empirical Bayes conformal prediction framework ($\CPr$) based on the $r$-value \cite{henderson-2015} by measuring how likely a candidate's latent score belongs to the top-ranked group after accounting for both its estimated score and its uncertainty. 
Theoretically, under exchangeability, using the $r$-value as the nonconformity score preserves the target conformal coverage, while it reduces the inclusion probability of high variance false candidates and leads to smaller expected prediction sets. We estimate this quantity either with a Normal–Normal empirical Bayes model, which gives a closed-form expression and theoretical insight, or with an assumption-free posterior-sampling estimator. 
The nonparametric version is a Monte Carlo estimator of the same posterior tail probability defining the theoretical $r$-value. 
Across image classification, VLM, and LLM tasks, $\CPr$ behaves similarly to $\CP$ or $\CPavg$ when variability is small, but produces smaller and more stable conformal sets when the variability is informative.

%% file: AISTATS/sections/preliminary.tex
We briefly review conformal prediction
and then show, with examples, why relying on a single pre-trained model, while neglecting model uncertainty, can impact $\CP$ reliability.

\textbf{Terminology/Notations.}
Consider a dataset $\{(x_i,y_i)\}_{i=1}^{n}$, where $x_i$ is an input such as an image, a question, or a prompt, and $y_i$ is the corresponding label or response. For a given input $x$, let $\mathcal{U}(x)=\{u_1,\ldots,u_K\}$ denote the candidate outputs, such as class labels in image classification or candidate responses in language tasks. The model, or an external evaluator, assigns each candidate $u_j$ a numerical score $f(x)_j$. For image classification, $f(x)_j$ can be the logit or probability of class $j$; for LLM tasks, it can be a likelihood or a quality score assigned by another model. We use the term "candidate" throughout for readability, which corresponds to the ``unit'' terminology in \cite{henderson-2015}.

\subsection{Conformal Prediction: Mechanism and Coverage}

$\CP$ constructs a prediction set by comparing test scores to a calibration threshold. Once a nonconformity score is fixed, $\CP$ provides a distribution free coverage guarantee under exchangeability. Here, the size of the prediction set depends on how this score ranks candidate outputs.

\textbf{Setup.}
Let $\{(x_i,y_i)\}_{i=1}^n$ be a calibration set. For each calibration example, we compute a nonconformity score $S_i=S(x_i,y_i)$ for the true output. For classification, one choice is $S_i = 1-\mathrm{softmax}\bigl(f(x_i)\bigr)_{y_i}$,
where smaller scores indicate more conforming labels. For a significance level $\alpha$, CP sets the threshold $B$ to be the $\lceil(n+1)(1-\alpha)\rceil$-th smallest calibration score.
For a test input $x_{\mathrm{test}}$, we compute $S(x_{\mathrm{test}},u_j)$ for every candidate $u_j\in\mathcal{U}(x_{\mathrm{test}})$ and include candidates whose scores are below the calibration threshold: $\mathcal{C}(x_{\mathrm{test}}) = \{u_j:S(x_{\mathrm{test}},u_j)<B\}.$ In classification, the set contains distinct labels (possible responses/answer options for language tasks).

\textbf{Coverage.}
Under exchangeability between calibration and test data, the rank of the true test score among the calibration scores is approximately uniform. Therefore, $\mathbb{P}\Bigl(y_{\mathrm{test}}\in \mathcal{C}(x_{\mathrm{test}})\Bigr)\ge 1-\alpha.$ This guarantee does not require a correctly specified model, but it does not by itself give small sets.
Set size depends on whether the nonconformity score ranks true candidates ahead of incorrect ones, which is where score variability becomes important.

\subsection{Impact of Variability on Conformal Set Selection}

Modern models often produce variable scores for the same or similar input. In vision models, this variability may come from posterior uncertainty over model parameters; in VLMs and LLMs, it may come from prompt paraphrasing or model based evaluation of candidate responses.

\textbf{Effect of variability.} To see why the variability matters, consider a toy setting with candidate logits
\[
p \sim \mathcal{N}(1,1) 
\quad \text{and} \quad 
q \sim \mathcal{N}(0,1000).
\]
Although $p$ has the larger mean, the variance of $q$ is so large that $q$ can exceed $p$ with probability close to one half ($\mathbb{P}(q > p) = 48.47\%$). Thus, a high variance candidate can occasionally appear highly confident even when its stable signal is weak. 
This is an issue for $\CP$ because $\CP$ observes only one realized score. More generally, let $g(x)$ denote the distribution of model scores for input $x$, and let $g_M(x)$ be a realization draw from $g(x)$. Since nonconformity scores are often functions of this realized score vector, variability in $g(x)$ can change candidate rankings and therefore alter the conformal set. Figure~\ref{fig:intro} shows this effect that different model instances sampled from the same approximate posterior can produce different conformal sets for the same image.
A natural alternative is $\CPavg$, which averages scores across multiple realizations \cite{Vovk2005} before applying $\CP$. Averaging, however, removes information about whether a high score is stable or noise-driven. To construct efficient conformal sets, we need a score that uses both the estimated value and its variability.

%% file: AISTATS/sections/R-value.tex
The above section identifies the main limitation of standard conformal rankings: $\CP$ uses one realized score, while $\CPavg$ averages realizations before calibration. However, neither approach directly uses score variability as a ranking signal. We instead treat the ranking step as an empirical Bayes problem and use the $r$-value as an uncertainty-aware nonconformity score. Here, we define the $r$-value, give parametric and assumption-free estimators, and show that the resulting conformal procedure preserves coverage while reducing set size under mild regularity conditions.

\textbf{Setup.}
For a fixed input $x$, suppose each candidate $u_i$ has an unobserved latent score $\theta_i$, interpreted as its stable signal, or long-run expected score, under the model or evaluation procedure. The observed score $f(x)_i$ is a noisy realization of this signal, with noise level allowed to vary across candidates. The $r$-value measures how likely $\theta_i$ is to lie in the top-ranked group after accounting for both the observed score and its uncertainty. We use this quantity as the nonconformity score in $\CP$.

\subsection{Parametric $r$-value under a Normal--Normal model}

We first present a parametric version of the $r$-value under a Normal--Normal empirical Bayes model. This model is mainly used to make the uncertainty-aware ranking explicit: it gives a closed-form expression for the $r$-value and describes how score magnitude and score variability trade off. We use it as an analytic approximation. For logit scores, approximately Gaussian fluctuations are plausible when posterior samples, bootstrap samples, or lightweight adapter perturbations satisfy standard posterior asymptotics (Bernstein--von Mises behavior). They are also consistent with Gaussian-process approximations to wide neural networks. For each candidate $u_i$, assume
\[
\theta_i \overset{\mathrm{iid}}{\sim} \mathcal{N}(\mu,\tau^2),
\qquad
f(x)_i \mid \theta_i \sim \mathcal{N}(\theta_i,\sigma_i^2).
\]
Here $\theta_i$ represents the candidate's stable latent score, while $\sigma_i^2$ measures the variability of the observed score across posterior or perturbation samples.

Under the Normal--Normal conjugate model, the posterior distribution of $\theta_i$ is
\begin{equation}
\label{equal:post}
    \theta_i \mid f(x)_i,\sigma_i^2
\sim
\mathcal{N}(\mu_{\theta_i},\sigma_{\theta_i}^2), \quad
\mu_{\theta_i}
=
\frac{\tau^2 f(x)_i+\sigma_i^2\mu}{\tau^2+\sigma_i^2},
\quad
\sigma_{\theta_i}^2
=
\frac{\tau^2\sigma_i^2}{\tau^2+\sigma_i^2}.
\end{equation}

The posterior mean shrinks the observed score toward the empirical Bayes center $\mu$, with stronger shrinkage for candidates whose scores have big variance. Thus, the posterior distribution in Eq.~\ref{equal:post} encodes both the estimated latent score and the uncertainty in that estimate. In practice, the hyperparameters $\mu$ and $\tau^2$ can be estimated from data by empirical Bayes.

For a fraction $\beta\in(0,1)$, let $\theta_\beta$ denote the $(1-\beta)$-quantile of the latent score distribution. The posterior tail probability
\[
V_\beta\bigl(f(x)_i,\sigma_i^2\bigr)
=
P\bigl\{\theta_i \ge \theta_\beta \mid f(x)_i,\sigma_i^2\bigr\}
\]
is the probability, after observing the noisy score and its variability, that candidate $i$ belongs to the top $\beta$ fraction of latent scores. Thus, $V_\beta$ is asking whether the observed score is large and whether the candidate is likely to remain near the top after uncertainty is taken into account.

Under the Normal--Normal model, the corresponding optimal threshold for selecting the top $\beta$ fraction has the closed form
\begin{equation}
\label{equal:threshold}
    t_\beta^*(\sigma_i^2)
=
\theta_\beta
\Bigl(1+\frac{\sigma_i^2}{\tau^2}\Bigr)
-
\mu\Bigl(\frac{\sigma_i^2}{\tau^2}\Bigr)
-
\frac{z_\beta\sigma_i\sqrt{\sigma_i^2+\tau^2}}{\tau},
\end{equation}
where $z_\beta$ is chosen so that the selected fraction is $\beta$. The key point is that the threshold depends on $\sigma_i^2$: two candidates with the same observed score can be ranked differently if one score is much less stable. The full derivation is given in Appendix~\ref{app:theorom}.

\subsection{From Posterior Tail Probability to $r$-values}

The quantity $V_\beta$ evaluates membership in a fixed top fraction $\beta$. The $r$-value combines these comparisons across all $\beta$ levels into a single ranking score. Let $\lambda_\beta=1-\Phi(z_\beta)$. Under the parametric threshold family, define
\[
r\bigl(f(x)_i,\sigma_i^2\bigr)
=
\inf
\Bigl\{
\beta:
V_\beta\bigl(f(x)_i,\sigma_i^2\bigr)
\ge
\lambda_\beta
\Bigr\}
=
\inf
\Bigl\{
\beta:
f(x)_i \ge t_\beta^*(\sigma_i^2)
\Bigr\}.
\]

A smaller $r$-value means that the candidate passes a more selective top-fraction threshold and is therefore ranked higher. Unlike ranking by $f(x)_i$ alone, the $r$-value ranks candidates by how confidently their latent scores belong near the top after accounting for score variability.

\subsection{Non-parametric estimation of $r$-values}  
The Normal--Normal model is useful for logit scores and theoretical analysis, but it is not appropriate for all score types. Probabilities are bounded in $[0,1]$, and VLM/LLM evaluator scores may be discrete, heavy-tailed, or non-Gaussian. In these cases, we estimate the same posterior tail probability $V_\beta$ directly from posterior or perturbation samples. Thus, the parametric and nonparametric constructions are two estimators of the same $r$-value quantity.

Suppose we have $M$ posterior or perturbation samples. For each posterior sample, we rank all $K$ candidates. For $\beta=k/K$, we estimate
\[
V_{k/K}(D_i)
\approx
\frac{1}{M}
\sum_{m=1}^{M}
\mathbf{1}
\{u_i \text{ appears among the top } k \text{ candidates in sample } m\},
\]
where $D_i$ denotes the posterior-sample information for candidate $i$. This frequency estimates how consistently candidate $i$ appears in the top $k$ positions across posterior samples. Candidates that are repeatedly ranked near the top receive small $r$-values, while candidates that appear near the top occasionally are penalized.

This assumption-free method is a Monte Carlo estimator of the same posterior tail probability $V_\beta$ that defines the theoretical $r$-value. Therefore, both estimators target the same uncertainty-aware ranking quantity. The parametric version gives closed form insight when the Gaussian approximation is appropriate, while the nonparametric version remains applicable without Normality assumptions.

\subsection{Coverage/Set‑Size Efficiency of CP with \texorpdfstring{$r$}{r}-values}

\begin{algorithm}[t]
\caption{Coverage Set Selection with $r$-value}
\label{alg:rvalue_subset}
{
\begin{algorithmic}[1]
\REQUIRE Calibration set $\{(x_i,y_i)\}_{i=1}^n$, test input $x_{\mathrm{new}}$, candidate outputs $\{u_1,\ldots,u_K\}$, significance level $\alpha$
\ENSURE Prediction set $C_r(x_{\mathrm{new}})$

\STATE Estimate score variability using posterior samples, such as WBB adapters, model ensembles, MC-dropout, or prompt paraphrases.
\STATE Compute the calibration $r$-values $r_i=r(x_i,y_i)$ for $i=1,\ldots,n$.
\STATE Set $\hat r$ to be the $\lceil(n+1)(1-\alpha)\rceil$-th smallest value among $\{r_i\}_{i=1}^n$.
\STATE For each candidate $u_j$ of $x_{\mathrm{new}}$, compute $r(x_{\mathrm{new}},u_j)$.
\STATE Return $C_r(x_{\mathrm{new}})=\{u_j:r(x_{\mathrm{new}},u_j)<\hat r\}$.
\end{algorithmic}
}
\end{algorithm}

\textbf{Coverage guarantee.} Let $\CPr$, $\CP$, and $\CPavg$ denote conformal prediction using $r$-values, standard CP, and average-then-CP, respectively. For each calibration example $(x_i,y_i)$, compute the $r$-value $r_i=r(x_i,y_i)$ of the true label. Set $\hat r$ to be the $\lceil(n+1)(1-\alpha)\rceil$-th smallest calibration $r$-value. For a new input $x_{\mathrm{new}}$ with candidates $\{u_1,\ldots,u_K\}$, define $C_r(x_{\mathrm{new}}) = \{u_j: r(x_{\mathrm{new}},u_j)<\hat r\}.$
Under exchangeability, $P\bigl(y_{\mathrm{new}}\notin C_r(x_{\mathrm{new}})\bigr) = P(r_{n+1}>\hat r) \le \alpha.$

\noindent
\begin{minipage}{\columnwidth}

\begin{theorem} \label{thm:main}
 Consider a new test input $x_{\text {new }}$. Let $C_r(x_{\mathrm{new}})$, $C_{\mathrm{std}}(x_{\mathrm{new}})$, and $C_{\mathrm{avg}}(x_{\mathrm{new}})$ denote the conformal sets produced by $\CPr$, $\CP$, and $\CPavg$, respectively.
For any false unit/label $u^{\prime}$, define the conditional inclusion probabilities:
\begin{align*}
P_{\mathrm{incl}}^{\mathrm{std}}\left(\sigma^2 \mid \mu_0\right)&=P\left(u^{\prime} \in C _{\mathrm{std}} \mid \theta_{u^{\prime}}=\mu_0, \sigma_{u^{\prime}}^2=\sigma^2, T_{\mathrm{std}}\right) \\
P_{\mathrm{incl}}^r\left(\sigma^2 \mid \mu_0\right)&=P\left(u^{\prime} \in C_r \mid \theta_{u^{\prime}}=\mu_0, \sigma_{u^{\prime}}^2=\sigma^2, r^*\right)
\end{align*}
where $T_{\text{std}}$ and $r^*$ are the $\frac{\lceil(n+1)(1-\alpha)\rceil}{n}$-quantiles of corresponding non-conformity scores of the calibration data, independent of  $x_{\text {new }}$. Then, 
\begin{compactenum}[\bfseries (1)]
\item If for any moderately small variance $\sigma^2$, we have $
P_{\mathrm{incl}}^r\left(\sigma^2 \mid \mu_0\right) \leq P_{\mathrm{incl}}^{\mathrm{std}}\left(\sigma^2 \mid \mu_0\right)
$, then $
\mathbb{E}\left[\left|C_r\right|\right] < \mathbb{E}\left[\left|C_{\text {std }}\right|\right]$.

\item There exists a constant $s\in[0,\infty)$ s.t. if all $\sigma_i^2 \ge s$ for $i = 1, \ldots, K$, then  
         $\mathbb{E}\!\bigl[\,|C_r|\,\bigr]
         \;<\;
         \mathbb{E}\!\bigl[\,|C_{\text{\rm avg}}|\,\bigr]$,
      where $C_{\text{\rm avg}}$ is obtained by setting $\sigma_i^{2} \rightarrow 0$.
\end{compactenum}
\end{theorem}
\end{minipage}

\textbf{Smaller set size.} The coverage guarantee follows from conformal calibration; the efficiency gain comes from the ranking induced by the $r$-value. Eq.~\ref{equal:post} converts the observed score $f(x)_i$ into a posterior distribution over the latent score $\theta_i$, so the ranking depends not only on the posterior mean $\mu_{\theta_i}$ but also on the posterior uncertainty $\sigma_{\theta_i}^2$. Eq.~\ref{equal:threshold} shows the same dependence from the threshold angle: the optimal cutoff varies with the candidate-specific variance $\sigma_i^2$. Thus, $\CPr$ uses the distribution of scores across posterior/perturbation samples ($\CPavg$ retains only the average).

When variability vanishes, the distinction between $\CPr$ and first-order conformal scores disappears. In the limit $\sigma_i^2\to 0$, Eq.~\ref{equal:post} gives $\mu_{\theta_i}\to f(x)_i$ and $\sigma_{\theta_i}^2\to 0$, while Eq.~\ref{equal:threshold} reduces to $t_\beta^*(0)=\theta_\beta$. Thus, the $r$-value reduces to a point-estimate ranking, matching the behavior of average-then-CP when many posterior samples are averaged. Conversely, when variability is non-negligible and informative, $\CPr$ uses it as a second-order signal: unstable false labels are penalized, while stable high-scoring candidates remain near the top. Under the conditions of Theorem~\ref{thm:main}(2), this yields smaller prediction sets under the same conformal coverage guarantee.

We see that $\CPr$ enjoys the same coverage property as $\CP$ and $\CPavg$, but it also returns smaller coverage sets. 
The remaining question is how to obtain the variability estimates in practice. Next, we instantiate the framework in two settings for vision and language models.

%% file: AISTATS/sections/Exp_Image.tex
\begin{figure}[t]
\
    \centering

    \begin{subfigure}[t]{0.48\linewidth}
        \centering
        \includegraphics[width=\linewidth]{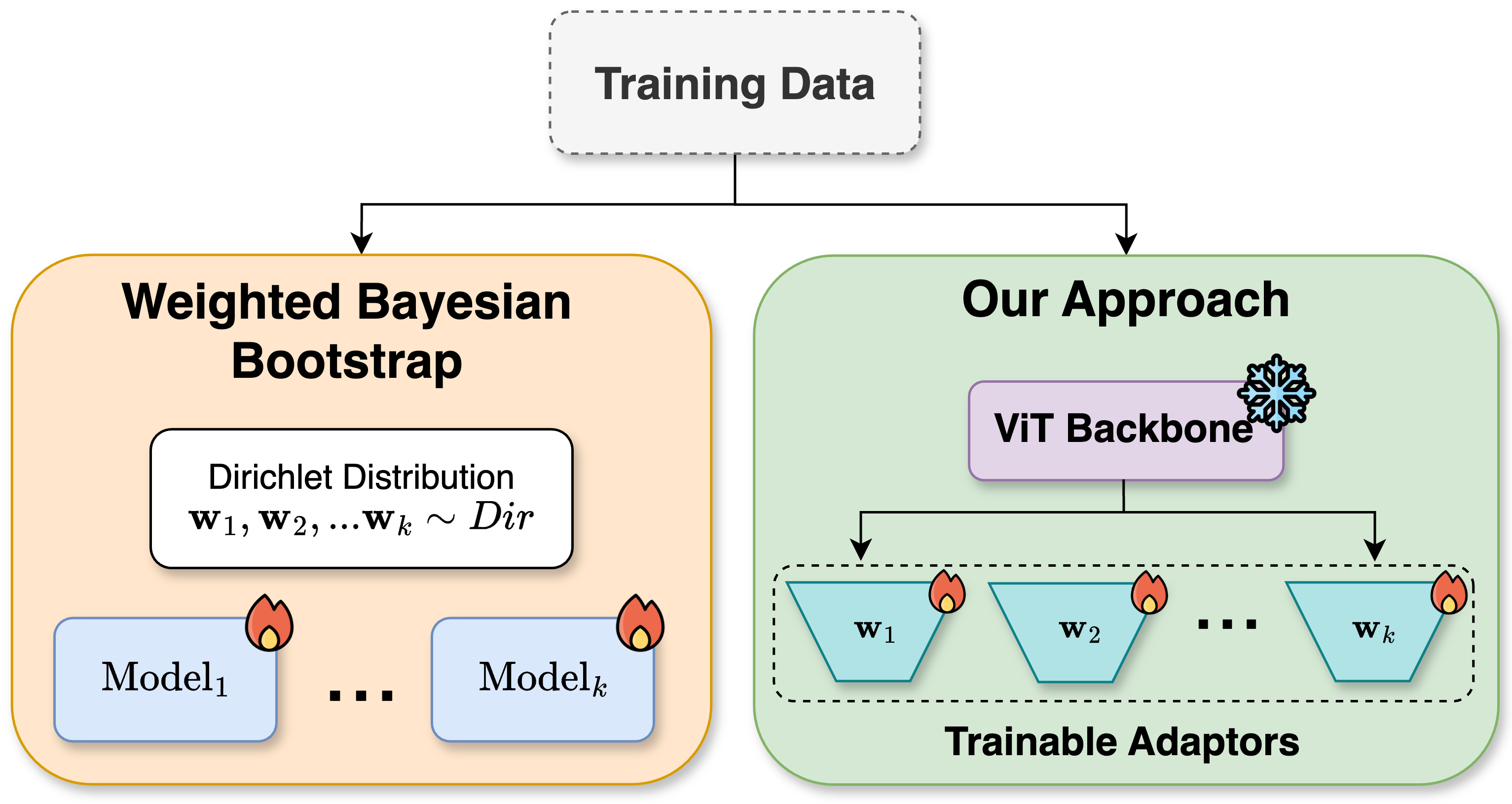}
        \label{fig:wbb}
    \end{subfigure}
    \hfill
    \begin{subfigure}[t]{0.48\linewidth}
        \centering
        \includegraphics[width=\linewidth]{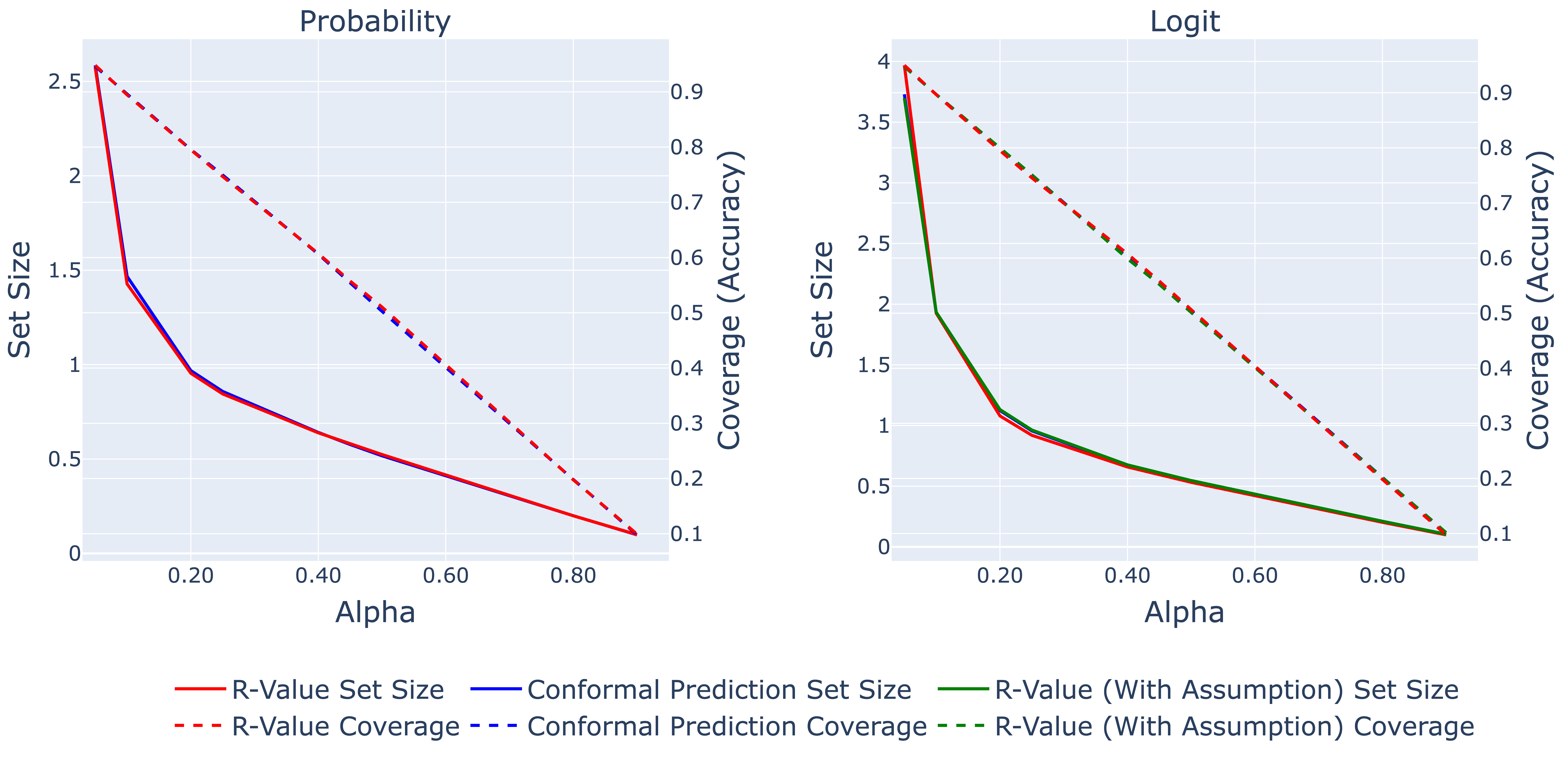}
        \label{fig:vitb_plots}
    \end{subfigure}
    \vspace{-7pt}
    \caption{\textbf{Left:} WBB approximates model uncertainty efficiently by training adapter modules \cite{houlsby-2019} instead of fully retraining the model. 
\textbf{Right:} Comparison of $\CP$ and $\CPr$ on ViT-Base image classification in probability and logit settings.}
    \label{fig:wbb_vitb}
    \vspace{-7pt}
\end{figure}

In this section, we study $\CPr$ in a vision setting where epistemic variability can be estimated from posterior model samples. Our goal is to verify the predictions of Section 3. When score variability is small, $\CPr$ should behave similarly to $\CP$ or $\CPavg$; when variability is informative, the $r$-value ranking should reduce unstable false label inclusion and produce smaller and more stable conformal sets.

{\bf Experiment questions.} Our experimental evaluations for image classification focus on the following questions
\begin{inparaenum}[\bfseries (a)]
    \item Does $\CPr$ preserve target coverage on standard vision backbones?
    \item Does $\CPr$ reduce to $\CP$-like behavior when variability is small?
    \item When multiple posterior samples are available, does $\CPr$ produce a more stable and efficient set than $\CP$ or $\CPavg$?
\end{inparaenum}

{\bf Setup.} We evaluate ImageNet classification using ViT-Large, ViT-Base, ResNet-50, and ResNet-18. For each input $x$ and candidate class $u_i$, the score $f(x)_i$ is either the pre-softmax logit or the post-softmax probability. For logits, we compute $r$-values using both the parametric and the nonparametric method. For probabilities, which are bounded in $[0,1]$, we use only the nonparametric estimator to avoid imposing a misspecified Gaussian model.

\textbf{Variability under Image Classification.} 
Figure~\ref{fig:wbb_vitb}(left) illustrates how we estimate posterior score variability. Models obtained by weighted bootstrapping approximate the posterior given the training data \cite{newton2020weighted}. However, repeatedly retraining large neural networks is expensive. We freeze the pretrained backbone and attach lightweight adapter modules following the adapter design of \cite{houlsby-2019}. We then train these adapters under the Weighted Bayesian Bootstrap (WBB) framework \cite{newton2020weighted}. Each WBB adapter provides one approximate posterior draw of the model parameters given the training data. The training cost remains modest because the adapters are lightweight and can be optimized in parallel. In our experiments, we train up to 1,800 ViT-Base adapters on a single A100 40GB GPU in about half an hour. Thus, WBB adapters provide a computationally feasible way to approximate posterior variability for large-scale models while keeping the pretrained model frozen.

{\em (A) $r$-value reduces to $\CP$ when variability is small.}
The ViT-Base results in Figure~\ref{fig:wbb_vitb}(right) show that $\CPr$ matches $\CP$ in both coverage and set size for probability and logit scores. This is a sanity check rather than a failure case. Here, WBB adapters estimate posterior score variability around a fixed pretrained backbone. The number of adapters controls the precision of the estimate, while the variability itself reflects how much the posterior samples perturb the model outputs. Since the pretrained backbone is strong and the lightweight adapters induce only modest perturbations, class-wise variability is small. As predicted by Theorem~\ref{thm:main}, when variability vanishes, the $r$-value reduces to a first order ranking based on the point estimate. Thus, $\CPr$ does not distort the usual CP ranking when there is little uncertainty to exploit. The same pattern appears for ViT-Large, ResNet-50, and ResNet-18 (see Appendix~\ref{app:vision}).

{\em (B) How does $\CPr$ differ?} 
Although $\CPr$ and $\CP$ can have similar aggregate coverage and set sizes, their rankings can differ on individual images. Figure~\ref{fig:image_comparisons}(left) shows examples where both methods include the true class but order the selected classes differently. This ranking matters because conformal sets are often used as ordered lists of plausible labels.

$\CP$ ranks classes using a single realized score, so a high but unstable score can outrank a more stable class. In contrast, $\CPr$ uses posterior score variability to measure how consistently each class remains near the top across WBB posterior samples. As a result, $\CPr$ can move the correct class upward; $\CPr$ often places the correct class higher in the coverage set (see Appendix~\ref{app:vision}). This illustrates the benefit of using variability as a ranking signal. 
\begin{figure}[t]
    \centering  
    \begin{subfigure}[t]{0.49\linewidth}
        \centering
        \includegraphics[width=\linewidth]{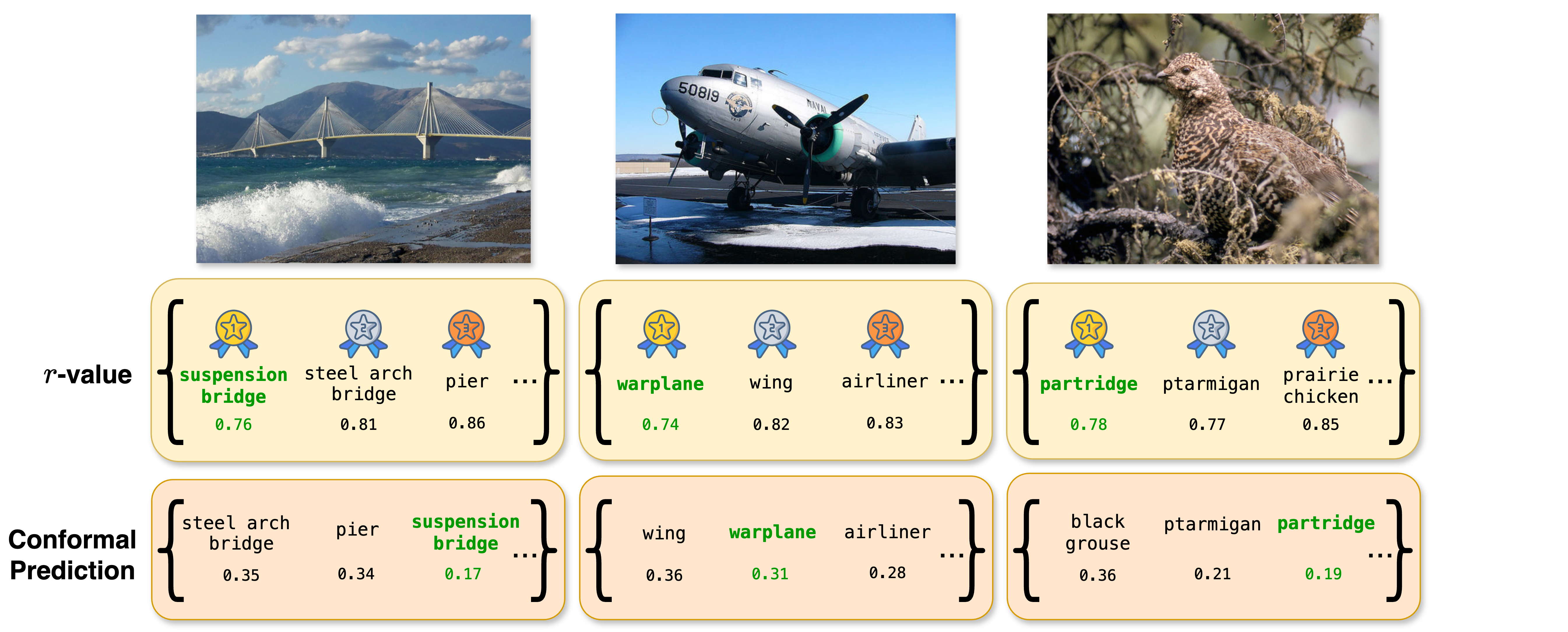}
        \label{fig:single_image_comparison}
    \end{subfigure}
    \hfill
    \begin{subfigure}[t]{0.49\linewidth}
        \centering
        \includegraphics[width=\linewidth]{AISTATS/figures/custom_image.png}
        \label{fig:multi_model}
    \end{subfigure}
    \vspace{-7pt}
    \caption{\textbf{Left:} Single-image comparison of $\CP$ and $\CPr$, where $\CPr$ incorporates model variability to often rank the correct class higher; smaller $r$-values are better. 
\textbf{Right:} Multi-model comparison showing that $\CPr$ produces smaller, more stable coverage sets than $\CP$. See Appendix~\ref{app:vision}.}
    \label{fig:image_comparisons}
\end{figure}

{\em (C) Informative variability leads to smaller and more stable sets.}
The advantage of $\CPr$ becomes pronounced when the variability across WBB posterior samples is informative. Applying $\CP$ separately to each posterior sample can produce different conformal sets for the same image, with different sizes and compositions (Figure~\ref{fig:image_comparisons}(right)). Since all posterior samples are plausible draws from the same approximate posterior, it is unclear which $\CP$ set should be trusted. $\CPavg$ reduces this ambiguity by averaging scores before calibration, but it also discards class specific variability. The $r$-value instead uses posterior samples as evidence of ranking stability. Rather than choosing one posterior sample or averaging all scores, $\CPr$ asks whether a candidate repeatedly appears among the top ranked classes. As shown in Figure~\ref{fig:image_comparisons}(right), $\CPr$ gives a more stable and smaller prediction set. This matches the mechanism in Section~3 that high variance false labels are less likely to survive the $r$-value ranking, while stable high latent score classes remain included.

\textbf{Summary.}
The vision experiments support two conclusions. First, $\CPr$ preserves the coverage behavior of $\CP$ that when WBB posterior variability is small, it behaves similarly to $\CP$ or $\CPavg$, as predicted by the zero variance limit. Second, the same posterior samples reveal that sample specific $\CP$ sets can vary in size and composition. Rather than choosing one posterior draw or averaging variability away, $\CPr$ integrates the samples into a single uncertainty aware ranking, yielding more stable sets and reducing unstable false-label inclusion. Detailed timing benchmarks are provided in Appendix~\ref{app:eff}.

%% file: AISTATS/sections/Exp_LLM.tex
\begin{figure*}[b]
    \centering
    \includegraphics[width=0.90\linewidth]{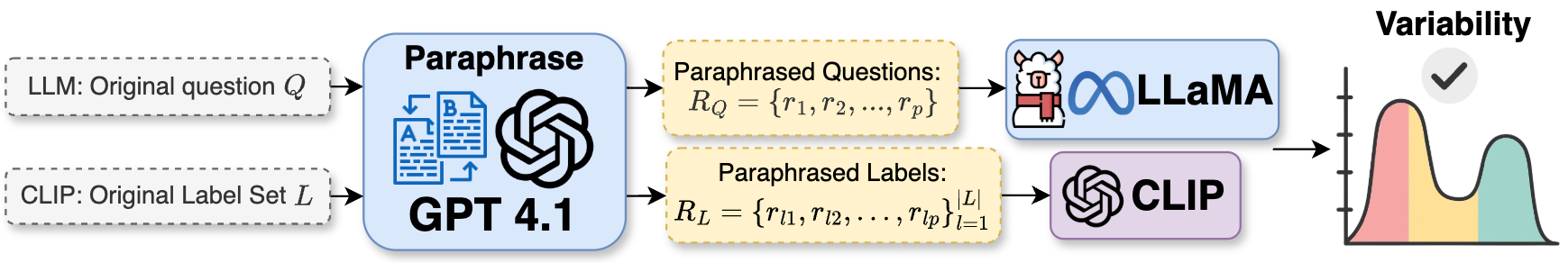}
    \vspace{-5pt}
    \caption{Pipeline for generating score variability via paraphrasing. For each discrete input, either an LLM prompt or a CLIP label description, we generate $p$ paraphrases using GPT-4.1. We then score each paraphrase (LLaMA: paraphrased prompt vs all candidate answers; CLIP: image vs. paraphrased labels), yielding a distribution of similarity scores reflecting the score's variability.}
    \vspace{-12pt}
    \label{fig:llm_setup}
\end{figure*}

Our experiments with VLMs and LLMs focus on the following questions:
\begin{inparaenum}[\bfseries (a)]
\item How can we quantify variability in language-related tasks?
\item When does $\CPr$ improve over $\CP$ and $\CPavg$ on VLM tasks?
\item How does the benefit of $\CPr$ change as model accuracy increases and epistemic variability decreases?
\item Can $\CPr$ be applied to LLMs, especially for open-ended response tasks? What are its limitations?
\end{inparaenum}

{\bf Setup.} We evaluate CLIPs on CIFAR-10, CIFAR-100, ImageNet, ImageNet-A, ImageNet-R, and EuroSAT on $\CP$, $\CPavg$, and $\CPr$. For CLIP Base-Patch16, we use calibration/test splits of 250/250 for CIFAR-10/100, 1500/1500 for ImageNet, and 500/500 for ImageNet-A/R and EuroSAT, repeated over 100 random splits. We also evaluate SigLIP2 and MobileCLIP2. In particular, MobileCLIP2 S2-S4 provide the same family sequence with increasing accuracy, allowing us to study how $\CPr$ behaves as models become stronger and more stable.

\begin{table*}[t]
  \centering
  \caption{Coverage and Set Size ($\alpha = 0.05$) on various datasets under CLIP.}
  \label{tab:CLIP}
  \vspace{-3pt}
  \resizebox{0.98\textwidth}{!}{
  \begin{tabular}{lcccccc}
    \toprule
    & \multicolumn{3}{c}{Coverage} & \multicolumn{3}{c}{Set Size} \\
    \cmidrule(lr){2-4} \cmidrule(lr){5-7}
    Dataset   & $\CPr$ & $\CPavg$ & $\CP$
              & $\CPr$ & $\CPavg$ & $\CP$ \\
    \midrule
    CIFAR-10  
      & $\mathbf{95.0\%\!\pm\!1.8\%}$ 
      & $95.7\%\!\pm\!1.7\%$ 
      & $95.4\%\!\pm\!1.8\%$
      & $\mathbf{1.3\!\pm\!0.1}$ 
      & $1.4\!\pm\!0.2$ {\color{purple}($-7.1\%$)} 
      & $1.4\!\pm\!0.2$ {\color{purple}($-7.1\%$)} \\

    CIFAR-100 
      & $\mathbf{95.2\%\!\pm\!1.9\%}$ 
      & $95.8\%\!\pm\!1.9\%$ 
      & $95.5\%\!\pm\!1.6\%$
      & $\mathbf{7.8\!\pm\!0.7}$ 
      & $8.1\!\pm\!1.9$ {\color{purple}($-3.7\%$)} 
      & $11.3\!\pm\!2.3$ {\color{purple}($-31.0\%$)} \\

    ImageNet 
      & $\mathbf{95.1\%\!\pm\!1.5\%}$ 
      & $95.2\%\!\pm\!1.8\%$ 
      & $95.6\%\!\pm\!1.9\%$
      & $\mathbf{5.7\!\pm\!0.9}$ 
      & $6.9\!\pm\!1.3$ {\color{purple}($-17.4\%$)} 
      & $10.1\!\pm\!1.7$ {\color{purple}($-43.6\%$)} \\

    ImageNet-A   
      & $\mathbf{94.8\%\!\pm\!1.6\%}$ 
      & $95.5\%\!\pm\!1.7\%$ 
      & $94.7\%\!\pm\!1.5\%$
      & $\mathbf{21.5\!\pm\!2.6}$ 
      & $23.5\!\pm\!3.8$ {\color{purple}($-8.5\%$)} 
      & $25.6\!\pm\!4.1$ {\color{purple}($-16.0\%$)} \\

    ImageNet-R  
      & $\mathbf{95.1\%\!\pm\!1.7\%}$ 
      & $95.4\%\!\pm\!1.9\%$ 
      & $95.6\%\!\pm\!1.9\%$
      & $\mathbf{3.6\!\pm\!0.4}$ 
      & $4.1\!\pm\!0.5$ {\color{purple}($-12.2\%$)} 
      & $4.6\!\pm\!0.4$ {\color{purple}($-21.7\%$)} \\

    EuroSAT 
      & $\mathbf{94.9\%\!\pm\!1.7\%}$ 
      & $94.7\%\!\pm\!1.5\%$ 
      & $95.2\%\!\pm\!1.2\%$
      & $\mathbf{4.9\!\pm\!0.2}$ 
      & $5.2\!\pm\!0.2$ {\color{purple}($-5.8\%$)} 
      & $5.6\!\pm\!0.3$ {\color{purple}($-12.5\%$)} \\
    \bottomrule
  \end{tabular}
  }
  \vspace{-7pt}
\end{table*}

\begin{table*}[b]
  \vspace{-7pt}
  \centering
  \caption{ImageNet coverage and set size at $\alpha=0.05$. Accuracy and coverage are reported in percent. Parentheses show relative set size reductions of $\CPr$ over each baseline.}
  \label{tab:imgnet_top20}
  \scriptsize
  \setlength{\tabcolsep}{2pt}
  \renewcommand{\arraystretch}{1.06}
  \begin{tabular}{@{}lccccccc@{}}
    \toprule
    & & \multicolumn{3}{c}{Coverage (\%)} & \multicolumn{3}{c}{Set size} \\
    \cmidrule(lr){3-5} \cmidrule(lr){6-8}
    Model & Accuracy & $\CPr$ & $\CPavg$ & $\CP$ & $\CPr$ & $\CPavg$ & $\CP$ \\
    \midrule
    SigLIP2-B/32 & $65.6$ & $\mathbf{94.9 \pm 1.3}$ & $95.1 \pm 1.4$ & $95.1 \pm 1.2$ & $\mathbf{13.0 \pm 1.6}$ & \cellrel{$14.4 \pm 2.6$}{-9.9} & \cellrel{$19.6 \pm 2.8$}{-33.5} \\
    SigLIP2-B/16 & $72.1$ & $\mathbf{94.9 \pm 1.3}$ & $95.0 \pm 1.2$ & $95.4 \pm 1.4$ & $\mathbf{7.3 \pm 1.6}$ & \cellrel{$8.3 \pm 1.3$}{-13.0} & \cellrel{$10.7 \pm 2.8$}{-32.3} \\
    SigLIP2-L/16 & $73.6$ & $\mathbf{94.9 \pm 1.4}$ & $95.3 \pm 1.2$ & $95.0 \pm 1.3$ & $\mathbf{8.8 \pm 1.6}$ & \cellrel{$9.8 \pm 1.9$}{-10.6} & \cellrel{$12.3 \pm 2.3$}{-28.5} \\
    SigLIP2-SO400M/16 & $75.1$ & $\mathbf{94.8 \pm 1.3}$ & $95.3 \pm 1.2$ & $95.3 \pm 1.3$ & $\mathbf{8.8 \pm 1.2}$ & \cellrel{$9.5 \pm 1.3$}{-7.2} & \cellrel{$10.9 \pm 2.1$}{-19.7} \\
    MobileCLIP2-S2 & $75.8$ & $\mathbf{95.0 \pm 1.2}$ & $95.1 \pm 1.4$ & $95.3 \pm 1.4$ & $\mathbf{4.0 \pm 0.5}$ & \cellrel{$4.4 \pm 0.6$}{-8.2} & \cellrel{$4.4 \pm 0.7$}{-8.0} \\
    MobileCLIP2-B & $76.8$ & $\mathbf{95.3 \pm 1.1}$ & $95.3 \pm 1.2$ & $95.2 \pm 1.3$ & $\mathbf{3.5 \pm 0.7}$ & \cellrel{$3.5 \pm 0.5$}{-0.5} & \cellrel{$3.6 \pm 0.8$}{-4.2} \\
    MobileCLIP2-S3 & $78.6$ & $\mathbf{95.1 \pm 1.2}$ & $95.4 \pm 1.4$ & $95.2 \pm 1.3$ & $\mathbf{3.1 \pm 0.4}$ & \cellrel{$3.2 \pm 0.4$}{-1.7} & \cellrel{$3.3 \pm 0.4$}{-6.1} \\
    MobileCLIP2-S4 & $79.5$ & $\mathbf{95.1 \pm 1.2}$ & $95.0 \pm 1.4$ & $95.1 \pm 1.4$ & $\mathbf{3.0 \pm 0.4}$ & \cellrel{$3.0 \pm 0.4$}{-0.4} & \cellrel{$3.1 \pm 0.7$}{-4.0} \\
    \bottomrule
  \end{tabular}
\end{table*}

{\em (A) Variability in VLMs and LLMs.} 
Variability in VLMs and LLMs differs from that in image models. In image classification, variability mainly comes from uncertainty in model training, whereas in VLMs and LLMs it also comes from how prompts are written. Since prompt format can significantly affect model responses~\cite{he-2024}, we estimate this variability by paraphrasing prompts or descriptions and then scoring each candidate across paraphrases. As shown in Figure~\ref{fig:llm_setup}, we use GPT-4.1 to generate multiple paraphrases for each prompt and evaluate candidates across these variations. The resulting score distribution captures prompt induced variability in model outputs. More details and examples of paraphrasing are provided in Appendix~\ref{app:llm} and Appendix~\ref{app:prompts}.

{\em (B) $\CPr$ helps when variability is informative.} 
Under CLIP, we compare $\CPr$ with $\CP$ and $\CPavg$ across datasets. 
Since CP coverage admits both lower and upper bounds~\cite{angelopoulos-2023}, coverage closer to $1-\alpha$ indicates less conservative calibration. 
Table~\ref{tab:CLIP} shows that $\CPr$ achieves near-target coverage while consistently reducing set size. 
On ImageNet, the average conformal set size under $\CPr$ is reduced by nearly half compared to $\CP$. 
These results indicate that variability provides useful ranking information, making $\CPr$ a viable alternative when such variability is available.

{\em (C) Gains shrink as accuracy increases.}
We next evaluate eight VLM variants spanning SigLIP2 and MobileCLIP2, as shown in Table~\ref{tab:imgnet_top20}. 
The results show a clear trend in efficiency gains that models with lower accuracy tend to benefit more from $\CPr$, while stronger models obtain smaller but still consistent reductions. 
This is consistent with Theorem~\ref{thm:main}. 
As models become more accurate and stable, conformal sets become smaller, leaving fewer unstable false labels for $\CPr$ to remove. 

This trend is visible in the percentage reductions reported in Table~\ref{tab:imgnet_top20}, full coverage and set size results are in Appendix~\ref{app:clip}.
Moving from the lowest accuracy model to the strongest model, default accuracy increases from $65.6\%$ to $79.5\%$, while the relative gain of $\CPr$ over $\CP$ decreases from $33.5\%$ to $4.0\%$. 
The gain over $\CPavg$ shows the same pattern, decreasing from $9.9\%$ to $0.4\%$. 
Although the trend is not perfectly monotone for every intermediate model, the overall trend is clear. Stronger models already produce smaller baseline conformal sets, leaving fewer unstable false labels for $\CPr$ to remove. 
This matches Theorem~\ref{thm:main} that $\CPr$ gives the largest gains when epistemic variability is informative and naturally approaches $\CP$-like behavior when variability is small.

{\em (D) LLM tasks.}
For close-ended LLM tasks, we evaluate MMLU using LLaMA 1B at $\alpha=0.05$. Since MMLU has only four answer options, ranking variability is limited, and $\CPr$ usually performs similarly to $\CPavg$. Gains appear in subjects with more variability. Average set size decreases from 2.865 to 2.754 for Marketing and from 3.422 to 3.307 for Clinical Knowledge. Because MMLU has only four answer options, the absolute gains are modest. This experiment is mainly a sanity check for applying $\CPr$ to discrete LLM candidate sets. Appendix~\ref{app:closed_llm} provides additional comparisons with Conformal Language Modeling~\cite{quach-2023} and further GPQA results.

For open-ended tasks, standard conformal sets are difficult to define because the output space is effectively unbounded and correctness is not binary. $\CPr$ does not remove this limitation, but it can improve ranking resolution. When evaluator scores are coarse, e.g., 0-10, many responses have tied scores. Variability across paraphrases or repeated evaluations helps break these ties and yields a more informative ordering (see Appendix~\ref{app:llm}).

\textbf{Summary.} Across VLM and LLM tasks, $\CPr$ preserves conformal coverage and improves efficiency when variability is informative. Its gains are large in higher variability regimes such as SigLIP2 and smaller for stronger, more stable models such as MobileCLIP2. This matches the theory that $\CPr$ reduces to $\CP$-like behavior when variability is negligible and removes unstable false candidates when variability is useful, making $\CPr$ a strong alternative to $\CP$.

%% file: AISTATS/sections/Relative_Work.tex
Conformal prediction provides distribution-free uncertainty quantification and has been applied to vision models, diffusion models, and LLMs \cite{Vovk2005, angelopoulos-2023, horwitz-2022, quach-2023, yadkori-2024}. We focus on work most related to improving conformal efficiency and using uncertainty information in conformal prediction.

\paragraph{Risk aware and ambiguous label CP.}
Risk controlling and conformal risk control methods extend CP from 0/1 miscoverage to user-defined losses or expected risk control \cite{angelopoulos-2021, angelopoulos-2024}. Separately, ambiguous label methods model uncertainty in the observed label caused by annotator disagreement or distribution valued labels \cite{javanmardi-2024, caprio-2024, stutz-2023}. These works address different sources of uncertainty. Our method keeps the standard marginal coverage target and treats the observed label as the calibration target while using epistemic variability to improve candidate ranking.

\paragraph{Set size optimization and aggregation.}
Other works reduce conformal set size by aggregating conformity scores or optimizing prediction set length under validity constraints \cite{luo-2025, kiyani-2024}. Recent ensemble conformal methods further improve efficiency by aggregating score vectors or prediction regions across multiple models \cite{ochoa-rivera-2025}. These approaches use multiple scores or multiple models to construct a stronger first order conformity score, or to select shorter valid sets after optimization. In contrast, our method does not learn aggregation weights, optimize a separate set size objective, or combine already constructed prediction sets. It uses posterior samples to estimate the stability of each candidate's rank and then uses the $r$-value as an uncertainty aware nonconformity score. Thus, variability is used as a second order ranking signal rather than being averaged before calibration.

\paragraph{Uncertainty aware conformal prediction.}
Several uncertainty aware CP methods use predictive distributions, local variance estimates, or ensemble outputs to adjust conformity scores \cite{nolte-2024, chernozhukov-2021, ochoa-rivera-2025}. These methods mainly target regression, continuous outcomes, or ensemble score aggregation. Recent rank based CP methods use model induced label orderings as conformity scores \cite{luo-2026}. Our method is related in that candidate ordering is central, but differs in how the ordering is obtained. Instead of relying on a single model induced rank, we estimate rank stability across posterior samples to penalize high scoring but unstable candidates. We use posterior tail probabilities to rank discrete candidates in classification, VLM, and LLM tasks, directly targeting ranking instability when labels or responses have similar scores but different epistemic stability.

%% file: AISTATS/sections/conclusion.tex
Our work introduces $\CPr$, an empirical Bayes conformal prediction framework for vision and language models. By using score variability as a second order signal, $\CPr$ preserves the $\CP$ coverage guarantee while producing more stable and efficient sets. When variability is small, $\CPr$ behaves similarly to $\CP$ or $\CPavg$; when variability is informative, it reduces unstable false candidates and yields smaller sets. This effect is especially pronounced in VLM tasks, where prompt induced variability provides useful ranking information. Our code will be publicly available.

\textbf{Limitations.}
The effectiveness of $\CPr$ depends on the quality of variability estimation. Adapter based WBB provides an efficient way to approximate posterior variability, but it may introduce a small drop in raw accuracy, especially for high performing models. In addition, the gains of $\CPr$ are largest when epistemic variability is non negligible. For highly stable models or tasks with few candidate options, the method may behave similarly to $\CPavg$ and yield only modest improvements.

%% file: AISTATS/sections/Appendix.tex
\section{Code Availability}
We provide the full implementation of the $r$-value method and demo for CLIP at:
\href{}{https://github.com/Yogesh914/conformal-rvalue}

\section{Computation Efficiency Analysis}
\label{app:eff}
\input{AISTATS/sections/Appendix/Efficiency}

\section{Proofs}
\label{app:theorom}
\input{AISTATS/sections/Appendix/Theorem}

\section{Additional Image Classification Results}
\label{app:vision}

\input{AISTATS/sections/Appendix/Supp_Vision}

\section{Additional VLM Results}
\label{app:clip}

\input{AISTATS/sections/Appendix/full_tables}

\section{Open-ended Response Tasks with LLMs}
\label{app:llm}
\input{AISTATS/sections/Appendix/Supp_LLM}

\section{Closed-ended Response Tasks with LLMs}
\label{app:closed_llm}
\input{AISTATS/sections/Appendix/Supp_LLM_Close}

\newpage
\section{Prompts}
\label{app:prompts}

\input{AISTATS/sections/Appendix/prompts}

\newpage
\section{Broader Impact}
\label{app:impact}
\input{AISTATS/sections/Appendix/Impact}

\clearpage

%% file: AISTATS/sections/Appendix/Efficiency.tex
\begin{table}[h]
\centering
\caption{Training time for 1,000 adapters across different models.}
\label{tab:train_time}
\begin{tabular}{lc}
\toprule
\textbf{Model} & \textbf{Training time} \\
\midrule
ResNet-18 & 20 min \\
ResNet-50 & 26 min \\
ViT-B     & 28 min \\
ViT-L     & 1 h 42 min \\
\bottomrule
\end{tabular}
\end{table}

\textbf{Efficient Analysis}
Estimating score variability for the $r$-value requires multiple model outputs and therefore introduces additional computation beyond standard $\CP$. This overhead has two main components. The first is the cost of obtaining posterior samples, such as WBB adapters or paraphrased prompts, which determines the precision of the estimated mean, variance, or empirical rank distribution. The second is the cost of computing the $r$-values and conformal sets over the calibration and test samples.

\textbf{Cost for posterior samples.} For vision models, the posterior-sampling cost remains modest because we only train lightweight adapters while keeping the pretrained backbone frozen. In our experiments, a single A100 40GB GPU can train up to 1,800 adapters across four vision backbones. For 1,000 adapters, ResNet-18, ResNet-50, and ViT-B each require less than 30 minutes, while ViT-L requires about 1.5 hours. Since adapters are lightweight and can be trained in parallel, this is substantially cheaper than repeatedly retraining the full model.

\textbf{Cost for paraphrasing} For the multi prompt paraphrasing approach in LLM experiments, the additional cost comes from two steps. Ones is generating paraphrased questions, and the other one is scoring candidate answers under each paraphrase. In our implementation, generating and evaluating 20 paraphrases for one question takes approximately 2-3 seconds with GPT-4.1. Although this is more expensive than a single forward pass, it is much smaller than the posterior sampling cost used in the vision experiments, where we train hundreds to thousands of WBB adapters. At the dataset level, the cost can be further reduced through batched inference over paraphrases and candidate answers. Thus, we view this overhead as a practical deployment consideration rather than a fundamental limitation of the method. Moreover, in many LLM workflows, repeated prompting, self consistency, or evaluator based scoring is already used. In such cases, the $r$-value can reuse these repeated outputs to estimate score variability instead of requiring a separate source of randomness.

\textbf{Cost for applying $r$-value.} The cost of computing conformal sets after the model outputs are obtained is also moderate. With 1,000 adapters and 10,000 ImageNet test samples, the parametric Normal-Normal implementation takes about 2 seconds to compute the $r$-values and conformal sets, which is comparable to standard CP. For the nonparametric estimator, computing the conformal set alone takes about 0.3 seconds, while computing both the nonparametric $r$-values and the resulting conformal sets takes about 1 minute. This additional cost comes from estimating rank consistency across posterior samples. Our current implementation is not optimized, so further speedups are possible through batching, vectorization, and parallel computation.

Overall, the additional computation is modest for vision models relative to the cost of posterior sample generation and is feasible in the settings we study. For VLMs and LLMs, multiple inference runs or paraphrased evaluations are increasingly standard in practice, and the $r$-value framework can use these existing repeated outputs to obtain uncertainty aware rankings rather than treating them only as averaged predictions.

%% file: AISTATS/sections/Appendix/Theorem.tex
\input{AISTATS/sections/Appendix/Threshold}

\input{AISTATS/sections/Theorem}

%% file: AISTATS/sections/Appendix/Threshold.tex
\subsection{Derivation of the Threshold Function under the Normal--Normal Model}

We acknowledge the work in \cite{henderson-2015}, which established that in the continuous model, a necessary condition for the function $t_\beta^*$ to be optimal within the class of continuously differentiable threshold functions is as follows. In the following work, we will derive the threshold function used in the parametric $r$-value construction. 
The purpose of this threshold is to decide, for a fixed top fraction $\beta$, whether a candidate with observed score $X_i$ and uncertainty $\sigma_i^2$ should be treated as belonging to the top $\beta$ fraction of latent scores. The key point is that the decision should not depend only on the observed score $X_i$, but also on how reliable that score is. A candidate with a large observed score but large variance should require stronger evidence than a candidate with a similar score and smaller variance.

Let
\[
\theta_i \sim \mathcal{N}(\mu,\tau^2),
\qquad
X_i \mid \theta_i,\sigma_i^2 \sim \mathcal{N}(\theta_i,\sigma_i^2).
\]
Here $\theta_i$ is the latent stable score of candidate $i$, while $X_i$ is the observed score. The variance $\sigma_i^2$ measures how much the observed score fluctuates around the latent score. For notational simplicity, write $s=\sigma_i^2$.

Under the Normal-Normal model, the posterior distribution of $\theta_i$ given $(X_i,s)$ is
\[
\theta_i \mid X_i,s
\sim
\mathcal{N}\{m(X_i,s),v(s)\},
\]
where
\[
m(X_i,s)
=
\frac{\tau^2 X_i+s\mu}{\tau^2+s},
\qquad
v(s)
=
\frac{\tau^2s}{\tau^2+s}.
\]
This posterior distribution is the object used by the $r$-value. Instead of asking whether the observed score $X_i$ is large, we ask whether the latent score $\theta_i$ is likely to lie in the top group after accounting for uncertainty.

Let $\theta_\beta$ denote the $(1-\beta)$-quantile of the prior distribution of $\theta_i$, so that
\[
P(\theta_i \ge \theta_\beta)=\beta.
\]
For a candidate with data $(X_i,s)$, define the posterior tail probability
\[
V_\beta(X_i,s)
=
P(\theta_i \ge \theta_\beta \mid X_i,s)
=
1-\Phi\left(
\frac{\theta_\beta-m(X_i,s)}{\sqrt{v(s)}}
\right).
\]
This quantity measures how likely candidate $i$ is to belong to the top $\beta$ fraction of latent scores.

Following the thresholding principle in \cite{henderson-2015}, the optimal boundary at level $\beta$ is characterized by a constant posterior tail probability along the boundary. That is, the threshold $t_\beta^*(s)$ is defined so that candidates exactly on the boundary satisfy
\[
P(\theta_i \ge \theta_\beta \mid X_i=t_\beta^*(s),s)
=
\lambda_\beta,
\]
where $\lambda_\beta$ is a level-specific constant chosen to satisfy the marginal selection constraint. Equivalently, if
\[
z_\beta=\Phi^{-1}(1-\lambda_\beta),
\]
then the boundary condition becomes
\[
\frac{\theta_\beta-m(t_\beta^*(s),s)}{\sqrt{v(s)}}=z_\beta.
\]
Substituting the posterior mean and variance gives
\[
\theta_\beta
-
\frac{\tau^2 t_\beta^*(s)+s\mu}{\tau^2+s}
=
z_\beta
\frac{\tau\sqrt{s}}{\sqrt{\tau^2+s}}.
\]
Solving this equation for $t_\beta^*(s)$ yields
\[
t_\beta^*(s)
=
\theta_\beta\left(1+\frac{s}{\tau^2}\right)
-
\mu\left(\frac{s}{\tau^2}\right)
-
z_\beta\frac{\sqrt{s}\sqrt{\tau^2+s}}{\tau}.
\]
Returning to $s=\sigma_i^2$, we obtain
\[
t_\beta^*(\sigma_i^2)
=
\theta_\beta\left(1+\frac{\sigma_i^2}{\tau^2}\right)
-
\mu\left(\frac{\sigma_i^2}{\tau^2}\right)
-
z_\beta
\frac{\sigma_i\sqrt{\tau^2+\sigma_i^2}}{\tau}.
\]

It remains to specify how $z_\beta$ is determined. The threshold should select a marginal fraction $\beta$ of candidates. Since
\[
X_i \mid s \sim \mathcal{N}(\mu,\tau^2+s),
\]
the marginal selection constraint is
\[
P\{X_i \ge t_\beta^*(s)\}
=
\int
\left[
1-\Phi\left(
\frac{t_\beta^*(s)-\mu}{\sqrt{\tau^2+s}}
\right)
\right]
g(s)\,ds
=
\beta,
\]
where $g(s)$ is the distribution of the variance $s=\sigma_i^2$. Substituting the expression for $t_\beta^*(s)$ gives
\[
\int
\left[
1-\Phi\left(
\frac{(\theta_\beta-\mu)\sqrt{\tau^2+s}}{\tau^2}
-
z_\beta\frac{\sqrt{s}}{\tau}
\right)
\right]
g(s)\,ds
=
\beta.
\]
This equation determines $z_\beta$, or equivalently $\lambda_\beta$.

Therefore, under the Normal-Normal model, the threshold function is
\[
t_\beta^*(\sigma_i^2)
=
\theta_\beta\left(1+\frac{\sigma_i^2}{\tau^2}\right)
-
\mu\left(\frac{\sigma_i^2}{\tau^2}\right)
-
z_\beta
\frac{\sigma_i\sqrt{\tau^2+\sigma_i^2}}{\tau},
\]
where $z_\beta$ is chosen so that the threshold selects a marginal fraction $\beta$ of candidates.
This expression shows explicitly how the selection boundary depends on both the observed score and its uncertainty. In particular, candidates with different variances are compared using different thresholds, which is the mechanism by which the $r$-value penalizes unstable high-score candidates.

%% file: AISTATS/sections/Theorem.tex
\subsection{Proof of Theorem 1}
We break Theorem 1 into the following components with distinct properties and subsidiary theorems, to provide clearer explanations and more detailed proofs. All the following proofs assume that $\theta_\alpha > 0$.

For a given test instance $x_{test}$, let $i$ index the $i$-th possible label. Then the prediction of the model for the $i$-th label is $f(x_{test})_i$ (denoted $f(x)_i$), which is a realization of the normal distribution $N(\theta_i, \sigma_i^2)$, where, WLOG, $\theta_i$ is the true unobserved quality parameter for label $i$ and is drawn from a prior distribution $N(0,1)$, and $\sigma_i^2$ is the variance of the prediction, which is also drawn from a prior distribution $g(\sigma^2)$. (Note that $g(\sigma^2)$ is not a point mass at $\sigma^2=0$ i.e., $P(\sigma^2 > 0) > 0$.) In general, $\theta_i \sim N(\mu,\tau)$ but we can standardize it to a standard normal distribution.

Recall that both methods aim to construct a conformal set that provides a statistical guarantee.

Let $C_r = \left\{i: r_i\leq r^* \right\} \equiv \left\{i: f(x)_i \ge t_{r^*}(\sigma_i^2) \right\} $ denote the conformal set under $r$-value, where nonconformity score for $i$-th class is defined as $r_i = \inf\{\alpha: f(x)_i \ge t_\alpha(\sigma_i^2)\}$ and $r^*$ is the $\frac{\lceil(n+1)(1-\alpha)\rceil}{n}$ quantile determined by calibration data and independent of $x_{test}$. Here $t_\alpha(\sigma_i^2) = \theta_\alpha\,(\sigma^2+1) - z_\alpha\,\sqrt{\sigma^4+\sigma^2}$ is a threshold function where $\theta_\alpha$ is the $1-\alpha$ - quantile of $N(0,1)$ and $z_\alpha$ is an $\alpha$ dependent constant that ensure $E_{\sigma_s^2 \sim g} \left[ \Phi\left(\theta_{\alpha}\sqrt{\sigma_s^2+1} - z_\alpha\sigma_s\right) \right] = 1-\alpha$. Similarly, let $C_{std} = \left\{i:  f(x)_i\ge T_{std} \right\}$ denote the conformal set under standard CP, where the non-conformity score for $i$-th class is usually defined as $-f(x)_i$ (smaller is better). And again, $T_{std}$ is the $\frac{\lceil(n+1)(1-\alpha)\rceil}{n}$ quantile determined by calibration data and independent of $x_{test}$. 

\begin{proposition}[Monotonicity of $t_\alpha$]
For a fixed $\alpha$, let 
\[
t_\alpha(\sigma^2)
   \;=\;
   \theta_\alpha\,(\sigma^2+1)
   \;-\;
   z_\alpha\,\sqrt{\sigma^4+\sigma^2},
\]
and define
\[
   s_\star \;=\;
   \frac{1}{2}\left(
       \frac{\theta_\alpha}{\sqrt{\theta_\alpha^{2}-z_\alpha^{2}}} - 1
   \right),
   \qquad
   s_0 \;=\;
   \frac{z_\alpha^{2}}{\theta_\alpha^{2}-z_\alpha^{2}} \quad \text{($s_0 > s_\star$)}.
\]

Then, if $0 < z_\alpha < \theta_{\alpha}$:
\begin{enumerate}
\item $t_\alpha$ is strictly decreasing on $[0, s_\star]$  
      and strictly increasing on $[s_\star, \infty)$ with respect to $\sigma^2$.
\item $t_\alpha(0) = t_\alpha(s_0) = \theta_\alpha$.
\end{enumerate}

If $z_\alpha \leq 0 < \theta_\alpha$:
\begin{enumerate}
\item $t_\alpha$ is strictly increasing on $[0, \infty)$.
\end{enumerate}
\end{proposition}

\begin{proof}

\textbf{(0) $z_\alpha$ is strictly less than $\theta_\alpha$}

Let $F(u) = \mathbb{E}_{\sigma^2 \sim g} \left[ \Phi\left(\theta_{\alpha}\sqrt{\sigma^2+1} - u\sigma\right) \right]$ for $u \in \mathbb{R}$. Then, by definition, $z_\alpha$ must satisfy the constraint such that $F(z_\alpha) = 1-\alpha$.

By the Dominated Convergence Theorem,
\[
\begin{aligned}
\frac{d F(u)}{d u} &= \frac{d}{d u} \mathbb{E}_{\sigma^2 \sim g}\left[\Phi\left(\theta_{\alpha} \sqrt{\sigma^2+1} - u \sigma\right)\right] \\
&= \mathbb{E}_{\sigma^2 \sim g}\left[\frac{d}{d u} \Phi\left(\theta_{\alpha} \sqrt{\sigma^2+1} - u \sigma\right)\right] \\
&= \mathbb{E}_{\sigma^2 \sim g}\left[\phi\left(\theta_{\alpha} \sqrt{\sigma^2+1} - u \sigma\right) \cdot (-\sigma)\right] < 0,
\end{aligned}
\]
where $\phi(\cdot)$ is the standard normal PDF.

Since $\frac{d F(u)}{d u} < 0$ for all $u$, it follows that $F(u)$ is strictly decreasing in $u$. Because $\Phi(\theta_\alpha) = 1 - \alpha$, we find that
\[
\begin{aligned}
F(\theta_\alpha) &= \mathbb{E}_{\sigma^2 \sim g} \left[ \Phi\left(\theta_{\alpha}\sqrt{\sigma^2+1} - \theta_\alpha\sigma\right) \right] \\
&= \mathbb{E}_{\sigma^2 \sim g} \left[ \Phi\left(\theta_{\alpha} \cdot \frac{1}{\sqrt{\sigma^2+1} + \sigma}\right) \right] \\
&< \mathbb{E}_{\sigma^2 \sim g} \left[ \Phi\left(\theta_{\alpha}\right) \right] = 1 - \alpha = F(z_\alpha),
\end{aligned}
\]
which implies that $\theta_\alpha > z_\alpha$ due to the monotonicity of $F(\cdot)$.

\textbf{(1) Monotonicity}

For notational simplicity, let $x = \sigma^2$ where $x \geq 0$. Then
\[
t_\alpha(x) = \theta_\alpha(x + 1) - z_\alpha \sqrt{x^2 + x}, \quad
t_\alpha'(x) = \theta_\alpha - z_\alpha \cdot \frac{2x + 1}{2\sqrt{x^2 + x}}.
\]

To study the monotonicity of $t_\alpha(x)$, we analyze its derivative.  
If $z_\alpha \leq 0$, then $t_\alpha'(x) > 0$ for all $x \ge 0$, since $\theta_\alpha > 0$ and the second term is non-positive. This implies that $t_\alpha(x)$ is strictly increasing on $[0, \infty)$, and no critical point exists.

If $\theta_\alpha > z_\alpha > 0$, set $t_\alpha'(x) = 0$. Then the critical point is:
\[
x = s_\star = \frac{1}{2} \left( \frac{\theta_\alpha}{\sqrt{\theta_\alpha^2 - z_\alpha^2}} - 1 \right).
\]

Since $z_\alpha < \theta_\alpha$, it follows that $(4\theta_\alpha^2 - 4z_\alpha^2) > 0$. The derivative $t_\alpha'(x)$ changes from negative to positive at $x = s_\star$, indicating that $t_\alpha$ achieves a minimum at $x = s_\star$. Hence, $t_\alpha$ is strictly decreasing on $[0, s_\star]$ and strictly increasing on $[s_\star, \infty)$.

\textbf{(2) Equality at endpoints}

We have shown that if $\theta_\alpha > z_\alpha > 0$, then $t_\alpha(\cdot)$ is strictly decreasing on $[0, s_\star]$ and strictly increasing on $[s_\star, \infty)$. We now show that $t_\alpha(s_0) = t_\alpha(0)$.

Since
\begin{align*}
t_\alpha(s_0) &= \theta_\alpha(s_0 + 1) - z_\alpha \sqrt{s_0^2 + s_0} \\
&= \theta_\alpha \cdot \frac{\theta_\alpha^2}{\theta_\alpha^2 - z_\alpha^2} - z_\alpha \cdot \frac{z_\alpha \theta_\alpha}{\theta_\alpha^2 - z_\alpha^2} \\
&= \frac{\theta_\alpha^3 - \theta_\alpha z_\alpha^2}{\theta_\alpha^2 - z_\alpha^2} = \theta_\alpha,
\end{align*}
and
\[
t_\alpha(0) = \theta_\alpha(0 + 1) - z_\alpha \cdot \sqrt{0 + 0} = \theta_\alpha,
\]
we conclude that $t_\alpha(s_0) = t_\alpha(0)$.

\end{proof}

\begin{proposition}[Monotone tail]
Define the probability that $f(x)_i$ surpasses the threshold function $t_\alpha(s)$ for any variance $s = \sigma_i^{2} \ge 0$ at a fixed level $\alpha$ as
\[
p_i(s) := P\bigl\{f(x)_i \ge t_\alpha(s)\bigr\}, \quad \text{for } i \in \{1, \dots, K\}.
\]
Under $t_\alpha(s)$, the mapping $s \mapsto p_i(s)$ is non-increasing on $[s_{\text{conj}}, \infty)$, where
\[
s_{\text{conj}} :=
\begin{cases}
\dfrac{z_\alpha^{2}}{\theta_\alpha^{2} - z_\alpha^{2}} & \text{if } 0 < z_\alpha < \theta_{\alpha}, \\[6pt]
0 & \text{if } z_\alpha \le 0 < \theta_\alpha.
\end{cases}
\]
Hence,
\begin{equation*}
p_i(s) \le p_i(0) \quad \text{for all } s \ge s_{\text{conj}}.
\end{equation*}
\end{proposition}

\begin{proof}
The result inherit from proposition 1.
\end{proof}

\begin{proposition}(Average-then-CP is the special case with variance 0)

Let $\,\bar f(x)_i = \frac{1}{M} \sum_{m=1}^{M} f(x)_i^{(m)}\,$ be the ensemble average of \emph{i.i.d.} model predictions $f(x)_i^{(m)} \sim \mathcal{N}(\theta_i, \sigma_i^{2})$  
for $i \in \{1, \dots, K\}$ and $m \in \{1, \dots, M\}$.

Define the two conformal prediction sets:
\[
C_{\mathrm{avg}} := \left\{ i : \bar f(x)_i \ge t_\alpha\left(\tfrac{\sigma_i^2}{M}\right) \right\},
\quad
C_r^{(0)} := \left\{ i : f(x)_i \ge t_\alpha\left(0\right) \right\}.
\]
Then,
\[
C_{\mathrm{avg}} = C_r^{(0)} \quad \text{almost surely.}
\]
\end{proposition}

\begin{proof}
We want to show that $C_{\mathrm{avg}} = C_{r}^{(0)}$ almost surely as $M \rightarrow \infty$.

Note that 
\[
f(x)_i^{(m)} \stackrel{\text{i.i.d.}}{\sim} \mathcal{N}(\theta_i, \sigma_i^2),
\quad \text{and} \quad
\bar f(x)_i = \frac{1}{M} \sum_{m=1}^{M} f(x)_i^{(m)}.
\]
Then, by properties of the normal distribution,
\[
\bar f(x)_i \sim \mathcal{N}\left(\theta_i, \frac{\sigma_i^2}{M}\right),
\quad \text{where} \quad
\mathrm{Var}(\bar f(x)_i) = \frac{\sigma_i^2}{M} \xrightarrow[M \to \infty]{} 0.
\]

By the strong law of large numbers, we have $\bar f(x)_i \xrightarrow{\text{a.s.}} \theta_i$.  
In particular, when $\sigma_i^2 = 0$, then $f(x)_i = \theta_i$ deterministically, and since $t_\alpha(0) = \theta_\alpha$, we conclude that for sufficiently large $M$,
\[
\bar f(x)_i \geq t_\alpha\left(\frac{\sigma_i^2}{M}\right)
\;\Longleftrightarrow\;
\theta_i \ge \theta_\alpha
\quad \text{almost surely.}
\]

This implies that the sets $C_{\mathrm{avg}}$ and $C_{r}^{(0)}$ are equivalent almost surely.
\end{proof}

\begin{theorem}[R-value selection]
\label{thm:r_to_avg}
Assume that $t_\alpha(\sigma_i^2) = \theta_\alpha\,(\sigma_i^2 + 1) - z_\alpha\,\sqrt{\sigma_i^4 + \sigma_i^2}$ and $\sigma_i^2 \ge s_{\text{conj}}$ for all classes $i$. 

Let $x_{\text{new}}$ be a new test input with unknown true label $y_{\text{true}}$. Let $C_r(x_{\text{new}})$ denote the conformal set selected by the $r$-value method, and let $C_{\text{avg}}(x_{\text{new}})$ denote the conformal set selected by the 'average-then-CP' method. When the number of models is sufficiently large, then:
\begin{enumerate}
\item \textbf{Coverage:} 
      \[
      P\left\{ y_{\text{true}} \notin C_r(x_{\text{new}}) \right\} \leq \alpha,
      \quad \text{and} \quad 
      P\left\{ y_{\text{true}} \notin C_{\text{avg}}(x_{\text{new}}) \right\} \leq \alpha.
      \]
\item \textbf{Set size:}
      \[
      \mathbb{E}\left[\,|C_r|\,\right] \;\le\; \mathbb{E}\left[\,|C_{\text{avg}}|\,\right],
      \]
      where $C_{\text{avg}}$ is defined by setting $\sigma_i^2 \to 0$ ('average-then-CP').
\end{enumerate}
\end{theorem}

\begin{proof}
1. The proof follows from the theory of conformal prediction.

2. For any selected conformal set $C$, the expected size of the set can be decomposed as:
\begin{align*}
\mathbb{E}[|C|] &= \mathbb{E}_X\left[\mathbf{1}\left(y_{\text{true}} \in C\right) + \sum_{y' \neq y_{\text{true}}} \mathbf{1}\left(y' \in C\right)\right] \\
&= \Pr\left(y_{\text{true}} \in C\right) + \sum_{y' \neq y_{\text{true}}} \Pr\left(y' \in C\right),
\end{align*}
where $\mathbf{1}(\cdot)$ is the indicator function and $y_{\text{true}}$ is the true label.

We want to show that
\[
\mathbb{E}[|C_r|] - \mathbb{E}[|C_{\text{avg}}|] 
= \left(P\left(y_{\text{true}} \in C_r\right) - P\left(y_{\text{true}} \in C_{\text{avg}}\right)\right) 
+ \sum_{y' \neq y_{\text{true}}} \left(P\left(y' \in C_r\right) - P\left(y' \in C_{\text{avg}}\right)\right)
\]
is non-positive.

From part (1), both $C_r$ and $C_{\text{avg}}$ satisfy marginal coverage guarantees:
\[
P\left(y_{\text{true}} \in C_r\right) \approx P\left(y_{\text{true}} \in C_{\text{avg}}\right).
\]
Hence, it suffices to show that
\[
\sum_{y' \neq y_{\text{true}}} \left(P\left(y' \in C_r\right) - P\left(y' \in C_{\text{avg}}\right)\right) \le 0.
\]

By assumption, $\sigma_i^2 \ge s_{\text{conj}}$, which satisfies the condition of Proposition 2. Then for any $y'$, we have:
\[
P\left(y' \in C_r\right) - P\left(y' \in C_{\text{avg}}\right) 
\overset{\text{Prop 3}}{=} 
P\left\{f(x)_{y'} \ge t_{r^*}(\sigma_{y'}^2)\right\} - P\left\{f(x)_{y'} \ge t_{r^*}(0)\right\} 
\overset{\text{Prop 2}}{\le} 0,
\]
where $r^*$ is the $r$-value threshold determined by the calibration data. Thus, the total difference is non-positive, and the result is proved.
\end{proof}

\begin{theorem}[Asymptotic rejection of high variance]
Assume the data-generating mechanism and notation introduced in the setup, and suppose $0 < \theta_\alpha < \infty$.

Let $y'$ be a non-true label whose latent quality is fixed at $\theta_{y'} = \mu_0$, where $\mu_0$ is sampled from $N(0, 1)$, and let $\sigma_{y'}^2 = \sigma^2 > 0$.

For any $\sigma^2$, define the conditional inclusion probabilities:
\[
\begin{aligned}
P_{\mathrm{incl}}^{\mathrm{std}}(\sigma^2 \mid \mu_0) &= P\left(y' \in C_{\mathrm{std}} \,\middle|\, \theta_{y'} = \mu_0,\; \sigma_{y'}^2 = \sigma^2,\; T_{\mathrm{std}} \right), \\
P_{\mathrm{incl}}^r(\sigma^2 \mid \mu_0) &= P\left(y' \in C_r \,\middle|\, \theta_{y'} = \mu_0,\; \sigma_{y'}^2 = \sigma^2,\; r^* \right),
\end{aligned}
\]
where $T_{\mathrm{std}}$ and $r^*$ are thresholds computed from calibration data independent of $x_{\text{test}}$, and are treated as fixed constants.

\begin{enumerate}
\item \textbf{Standard CP:} 
      \[
      \lim_{\sigma \rightarrow \infty} P_{\mathrm{incl}}^{\mathrm{std}}(\sigma^2 \mid \mu_0) = \frac{1}{2}.
      \]
\item \textbf{$r$-value CP:}
      \[
      \lim_{\sigma \rightarrow \infty} P_{\mathrm{incl}}^r(\sigma^2 \mid \mu_0) = 0.
      \]
\end{enumerate}
\end{theorem}

\begin{proof}
For any label $y'$ such that $f(x)_{y'} \sim \mathcal{N}(\mu_0, \sigma^2)$, by definition, the probability that the false label $y'$ is included in $C_{\text{std}}$ is:
\begin{align*}
P_{\text{incl}}^{\text{std}}\left(\sigma^2 \mid \mu_0\right) 
&= P\left(y' \in C_{\text{std}} \,\middle|\, \theta_{y'} = \mu_0,\, \sigma_{y'}^2 = \sigma^2,\, T_{\text{std}} \right) \\
&= P\left( f(x)_{y'} > T_{\text{std}} \mid \mu_0, \sigma^2, T_{\text{std}} \right) 
= 1 - \Phi\left( \frac{T_{\text{std}} - \mu_0}{\sigma} \right).
\end{align*}

Similarly, the probability that $y'$ is included in $C_r$ is:
\begin{align*}
P_{\text{incl}}^r\left(\sigma^2 \mid \mu_0\right) 
&= P\left(y' \in C_r \,\middle|\, \theta_{y'} = \mu_0,\, \sigma_{y'}^2 = \sigma^2,\, r^* \right) \\
&= P\left( \frac{f(x)_{y'} - \mu_0}{\sigma} \ge \frac{t_{r^*}(\sigma^2) - \mu_0}{\sigma} \right) 
= 1 - \Phi\left( \frac{t_{r^*}(\sigma^2) - \mu_0}{\sigma} \right).
\end{align*}

Define 
\[
A_{\text{std}}(\sigma^2) := \frac{T_{\text{std}} - \mu_0}{\sigma}, \quad 
A_r(\sigma^2) := \frac{t_{r^*}(\sigma^2) - \mu_0}{\sigma}.
\]
Then the inclusion probabilities are functions of $\sigma$ and become:
\begin{align}
P_{\text{incl}}^{\text{std}}(\sigma^2 \mid \mu_0) &= 1 - \Phi\left( A_{\text{std}}(\sigma^2) \right), \\
P_{\text{incl}}^r(\sigma^2 \mid \mu_0) &= 1 - \Phi\left( A_r(\sigma^2) \right).
\end{align}

Note that $T_{\text{std}}$ and $\mu_0$ are constants. However, the behavior of $t_{r^*}(\sigma^2)$ as a function of $\sigma$ is unclear. Recall:
\[
t_{r^*}(\sigma^2) = \theta_{r^*}(\sigma^2 + 1) - z_{r^*} \sqrt{\sigma^4 + \sigma^2}.
\]
By Taylor expansion, $\sqrt{1 + \frac{1}{\sigma^2}} = 1 + \frac{1}{2\sigma^2} + o\left( \frac{1}{\sigma^2} \right)$, then we get:
\begin{align*}
A_r(\sigma^2) 
&= \frac{\theta_{r^*}(\sigma^2 + 1) - z_{r^*} \sqrt{\sigma^4 + \sigma^2} - \mu_0}{\sigma} \\
&= \theta_{r^*}\left(\sigma + \frac{1}{\sigma} \right) - z_{r^*} \sigma \left( 1 + \frac{1}{2\sigma^2} + o\left( \frac{1}{\sigma^2} \right) \right) - \frac{\mu_0}{\sigma} \\
&= (\theta_{r^*} - z_{r^*}) \sigma + \frac{1}{\sigma} \left( \theta_{r^*} - \frac{z_{r^*}}{2} - \mu_0 \right) + o\left( \frac{1}{\sigma} \right).
\end{align*}

From Proposition 1, we know that $\theta_{r^*} - z_{r^*} > 0$. Therefore,
\begin{align*}
\lim_{\sigma \to \infty} P_{\text{incl}}^{\text{std}}(\sigma^2 \mid \mu_0) 
&= 1 - \Phi\left( \lim_{\sigma \to \infty} \frac{T_{\text{std}} - \mu_0}{\sigma} \right) 
= 1 - \Phi(0) = \frac{1}{2}, \\
\lim_{\sigma \to \infty} P_{\text{incl}}^r(\sigma^2 \mid \mu_0) 
&= 1 - \Phi\left( \lim_{\sigma \to \infty} A_r(\sigma^2) \right) 
= 1 - \Phi(\infty) = 0.
\end{align*}
Thus proved.
\end{proof}

\begin{theorem} 
\label{thm:r_to_std}
Let the data-generating assumptions and notation of Theorem 2 hold. Suppose, in addition, that:

\begin{enumerate}
\item \textbf{Uniform dominance of inclusion probabilities.}  
For every moderately small $\sigma^2 \geq 0$ and every false label $y'$ with $\theta_{y'} = \mu_0$,
\[
P_{\mathrm{incl}}^r\left(\sigma^2 \mid \mu_0\right) \leq P_{\mathrm{incl}}^{\mathrm{std}}\left(\sigma^2 \mid \mu_0\right).
\]
\end{enumerate}

Then the $r$-value conformal predictor produces, on average, a smaller prediction set than standard conformal prediction:
\[
\mathbb{E}\left[\,|C_r|\,\right] \le \mathbb{E}\left[\,|C_{\mathrm{std}}|\,\right].
\]
\end{theorem}

\begin{proof}
Let $Y_{\mathrm{true}}$ denote the true label, and let $y' \ne Y_{\mathrm{true}}$ denote any false label. The expected size of a conformal prediction set $C$ can be decomposed as:
\[
\mathbb{E}[|C|] = \mathbb{E}[\mathbf{1}\{Y_{\mathrm{true}} \in C\}] + \sum_{y' \ne Y_{\mathrm{true}}} \mathbb{E}[\mathbf{1}\{y' \in C\}] 
= \mathbb{P}(Y_{\mathrm{true}} \in C) + \sum_{y' \ne Y_{\mathrm{true}}} \mathbb{P}(y' \in C).
\]

Since both $C_r$ and $C_{\mathrm{std}}$ guarantee marginal coverage, we have:
\[
\mathbb{P}(Y_{\mathrm{true}} \in C_r) = \mathbb{P}(Y_{\mathrm{true}} \in C_{\mathrm{std}}) \geq 1 - \alpha.
\]

Therefore, following the same idea of proof of Theorem 2, the difference in expected set sizes is governed by the false label terms:
\[
\mathbb{E}[|C_r|] - \mathbb{E}[|C_{\mathrm{std}}|] = \sum_{y' \ne Y_{\mathrm{true}}} \left[ \mathbb{P}(y' \in C_r) - \mathbb{P}(y' \in C_{\mathrm{std}}) \right].
\]

Let $p(\sigma^2)$ be the probability density function (pdf) of $\sigma^2$, and let $\mu \sim \mathcal{N}(0,1)$ with pdf denoted by $q(\mu)$. Then for any conformal set $C$,
\[
\mathbb{P}(y' \in C) = \int_{-\infty}^\infty \int_0^\infty P_{\mathrm{incl}}^C(\sigma^2 \mid \mu) \cdot p(\sigma^2)\, q(\mu)\, d\sigma^2\, d\mu.
\]

Define the difference of inclusion probabilities:
\[
\Delta P(\sigma^2 \mid \mu) := P_{\mathrm{incl}}^r(\sigma^2 \mid \mu) - P_{\mathrm{incl}}^{\mathrm{std}}(\sigma^2 \mid \mu).
\]
Then,
\[
\mathbb{P}(y' \in C_r) - \mathbb{P}(y' \in C_{\mathrm{std}}) = \int_{-\infty}^\infty \int_0^\infty \Delta P(\sigma^2 \mid \mu) \cdot p(\sigma^2)\, q(\mu)\, d\sigma^2\, d\mu.
\]

Theorem 3 implies that $\Delta P(\sigma^2 \mid \mu) < 0$ for sufficiently large $\sigma^2$. Moreover, by the assumption, we have $\Delta P(\sigma^2 \mid \mu) \leq 0$ for moderately small $\sigma^2$. Putting both together, we conclude that:
\[
\Delta P(\sigma^2 \mid \mu) \leq 0 \quad \text{for all } \sigma^2.
\]

Hence,
\[
\mathbb{E}[|C_r|] \le \mathbb{E}[|C_{\mathrm{std}}|],
\]
as desired.

Furthermore, if we make the additional assumption that $p(\sigma^2) > 0$ on any interval $[0, \infty)$, then the inequality is strictly negative:
\[
\mathbb{P}(y' \in C_r) - \mathbb{P}(y' \in C_{\mathrm{std}}) < 0,
\]
and we have:
\[
\mathbb{E}[|C_r|] < \mathbb{E}[|C_{\mathrm{std}}|].
\]
\end{proof}

In the following, we present two extensions that relax the assumptions required in our main result. Both results allow a proportion of variances to violate the original assumptions while still guaranteeing the efficiency of the $\CPr$ method.
We start by showing that $\CPr$ does not produce a larger conformal set in expectation than $\CPavg$ if the number of high-variance labels is large. This result becomes stronger as the number of possible false labels increases.

\begin{theorem}[R-value selection under weaker assumption]
\label{thm:variance-efficient-weak}
Following the notation of Theorem 2.
Let
\[
\mathbf{H} := \bigl\{i : \sigma_i^2 > s_{\mathrm{conj}}\bigr\},
\qquad
\mathbf{L} := \bigl\{i : \sigma_i^2 \le s_{\mathrm{conj}}\bigr\},
\qquad
k := |\mathbf{H}| + |\mathbf{L}|
\]
denote the sets of high- and low-variance labels, respectively, for a new test input, where $k$ is the total number of possible false labels.
For each label $i$, let $p_i^{\mathrm{avg}}$ and $p_i^{r}$ denote the inclusion probabilities under $\CPavg$ and $\CPr$, respectively. Define
\[
\Delta := \min_{i \in \mathbf{H}}\!\bigl(p_i^{\mathrm{avg}} - p_i^{r}\bigr),
\qquad
\varepsilon := \max_{i \in \mathbf{L}}\!\bigl(p_i^{r} - p_i^{\mathrm{avg}} \bigr),
\qquad
f(\delta) := \Delta(1 - \delta) - \varepsilon\delta,
\]
where $\delta \ge \frac{|\mathbf{L}|}{k}$ is an upper bound on the proportion of low-variance labels.

\begin{enumerate}
\item \textbf{Expectation-level guarantee.}
      If
      \[
            \delta \le \frac{\Delta}{\Delta + \varepsilon}
            \quad\Longleftrightarrow\quad
            \frac{|\mathbf{H}|}{|\mathbf{L}|} \ge \frac{\varepsilon}{\Delta},
      \]
      then, for all $k \ge 1$,
      \[
           \mathbb{E}\bigl[|C_{r}|\bigr]
           \le
           \mathbb{E}\bigl[|C_{\mathrm{avg}}|\bigr].
      \]

\item \textbf{High-probability guarantee.}
      Fix any confidence level $\eta \in (0, 1)$.  
      If
      \[
            k \ge \frac{2 \log(1/\eta)}{f(\delta)^2},
      \]
      then with probability at least $1 - \eta$,
      \[
            |C_{r}| < |C_{\mathrm{avg}}|.
      \]
\end{enumerate}
\end{theorem}

\begin{proof}
Since at most a $\delta$ fraction of labels have variance less than or equal to $s_{\text{conj}}$, we have set sizes $|\mathbf{L}| \le \delta k$ and $|\mathbf{H}| \ge (1 - \delta)k$.

Define
\[
G_k := \sum_{i=1}^k \left[ \mathbf{1}(i \in C_{\mathrm{avg}}) - \mathbf{1}(i \in C_r) \right],
\]
as the difference of number of labels selected in two sets, and let $p_i^{\mathrm{avg}}$ and $p_i^r$ denote the probabilities that label $i$ is included in $C_{\mathrm{avg}}$ and $C_r$, respectively. Then,
\[
\mathbb{E}[G_k] = \sum_{i=1}^k \left( p_i^{\mathrm{avg}} - p_i^r \right) = \sum_{i \in \mathbf{H}} (p_i^{\mathrm{avg}} - p_i^r) + \sum_{i \in \mathbf{L}} (p_i^{\mathrm{avg}} - p_i^r).
\]

For $i \in \mathbf{H}$, Proposition 2 guarantees $p_i^{\mathrm{avg}} - p_i^r \ge \Delta  = \min_{i \in \mathbf{H}}\!\bigl(p_i^{\mathrm{avg}} - p_i^{r}\bigr) \ge 0$. Therefore,
\[
\sum_{i \in \mathbf{H}} (p_i^{\mathrm{avg}} - p_i^r) \ge \Delta |\mathbf{H}| \ge \Delta (1 - \delta)k.
\]

For $i \in \mathbf{L}$, by the Extreme Value Theorem, there exists $\varepsilon = \max_{i \in \mathbf{L}}\!\bigl( p_i^{r} - p_i^{\mathrm{avg}}\bigr)> 0$ such that
\[
|p_i^{\mathrm{avg}} - p_i^r| \le \varepsilon
\quad \Rightarrow \quad
\sum_{i \in \mathbf{L}} (p_i^{\mathrm{avg}} - p_i^r) \ge -\varepsilon |\mathbf{L}| \ge -\varepsilon \delta k.
\]

Combining both bounds:
\[
\mathbb{E}[G_k] \ge \Delta |H| - \epsilon|L| = \Delta (1 - \delta)k - \varepsilon \delta k = k \cdot f(\delta),
\quad \text{where } f(\delta) := \Delta(1 - \delta) - \varepsilon \delta.
\]

\begin{enumerate}
\item[\textbf{1.}] \textbf{Expectation-level guarantee.}  
If $f(\delta) \ge 0$, then $\mathbb{E}[G_k] \ge 0$. This implies
\[
\mathbb{E}\left[|C_{\mathrm{avg}}|\right] \ge \mathbb{E}\left[|C_r|\right].
\]

\item[\textbf{2.}] \textbf{High-probability guarantee.}  
Note that for each $i$, the difference $\mathbf{1}(i \in C_{\mathrm{avg}}) - \mathbf{1}(i \in C_r)$ lies in $[-1, 1]$. Since these are independent across $i$, we may apply Hoeffding's inequality:
\begin{align*}
\mathbb{P}(G_k \le 0) &= \mathbb{P}\left( G_k - \mathbb{E}[G_k] \le -\mathbb{E}[G_k] \right) \\
&\le \exp\left( -\frac{2\,\mathbb{E}[G_k]^2}{4k} \right) \\
&\le \exp\left( -\frac{2 (k f(\delta))^2}{4k} \right)
= \exp\left( -\frac{1}{2}k f(\delta)^2 \right).
\end{align*}

To ensure
\[
\mathbb{P}\left( |C_r| < |C_{\mathrm{avg}}| \right) \ge 1 - \eta,
\]
we require:
\[
\exp\left( -\frac{1}{2}k f(\delta)^2 \right) \le \eta \quad \Longleftrightarrow \quad k \ge \frac{2\log(1/\eta)}{ f(\delta)^2}.
\]

Thus, the condition in the theorem statement guarantees the result.
\end{enumerate}
\end{proof}

\textbf{Interpretation.}
When high-variance labels outnumber low-variance by $\varepsilon/\Delta$, $\CPr$ always produce a conformal set that is not larger than $\CPavg$ in expectation for any $k$.
The probability of $\CPr$ generating a strictly smaller conformal set than $\CPavg$ converges to 1 at an exponential rate in $k$, which matches the empirical finding that $\CPr$ perform better with more possible labels where variability becomes an important signal (see Table 1).

We then show that even if $\CPr$ is only on average less likely to include false labels than $\CP$, it still guarantees a smaller expected conformal set.

\begin{theorem}[Weaker assumption for Theorem 4]
\label{thm:weaker}
Follow the setup and notation of Theorem 4.  
For every false label $y'$ with latent mean $\theta_{y'}=\mu_0$, let
\[
\Delta P(\sigma^{2}\mid\mu_0)
   \;=\;
   P_{\mathrm{incl}}^{r}\!\left(\sigma^{2}\mid\mu_0\right)
   \;-\;
   P_{\mathrm{incl}}^{\mathrm{std}}\!\left(\sigma^{2}\mid\mu_0\right),
\quad
p(\sigma^{2})
\]
denote the inclusion-probability gap and the pdf of the variance
$\sigma^{2}$.  
Suppose that the average gap is negative:
\begin{equation}
\tag{A}
\int_{0}^{\infty}\!\Delta P(\sigma^{2}\mid\mu_0)\;p(\sigma^{2})\,d\sigma^{2}
\;<\;0 .
\end{equation}
Then the $r$-value conformal predictor yields a smaller expected
prediction set than standard conformal prediction:
\[
\mathbb{E}\!\left[\,|C_{r}|\,\right]
   \;\le\; 
\mathbb{E}\!\left[\,|C_{\mathrm{std}}|\,\right].
\]
\end{theorem}

\begin{proof}
Let $\Delta E:=\mathbb{E}[|C_r|]-\mathbb{E}[|C_{\mathrm{std}}|]$.
Because both methods have the same marginal coverage for the true label, we follow the same idea of proof of Theorem 4. The difference arises only from false labels:
\[
\int_{-\infty}^\infty \int_0^\infty \Delta P(\sigma^2 \mid \mu) \cdot p(\sigma^2)\, q(\mu)\, d\sigma^2\, d\mu,
\] which is less than or equal to 0 due to Assumption A and $p(\sigma^2)$ is nonnegative, which implies that \[ \mathbb{E}[|C_r|] \le \mathbb{E}[|C_{\mathrm{std}}|].\]
Similarly, if we further assume that $p(\sigma^2) > 0$ on any interval $[0, \infty)$, then the inequality is strictly negative and we have
\[
\mathbb{E}[|C_r|] < \mathbb{E}[|C_{\mathrm{std}}|].
\]
\end{proof}

%% file: AISTATS/sections/Appendix/Supp_Vision.tex
\textbf{Comparable performance in the single-backbone setting.}
In Figure~\ref{fig:conformal-like-coverage_supp}, we compare standard conformal prediction and $\CPr$ on ImageNet using ResNet18, ResNet50, ViT-Base, and ViT-Large. We evaluate both probability and logit settings and report coverage and set size across different significance levels. In most cases, $\CPr$ achieves performance comparable to standard CP. This behavior is expected in the single backbone setting. Although we use WBB adapters to obtain posterior samples around the pretrained backbone, the backbone itself is frozen, and the adapters are lightweight. Therefore, the induced variability is relatively modest. When the variability signal is small, the $r$-value ranking naturally becomes close to the usual first order ranking used by CP or $\CPavg$. Thus, the similarity between $\CPr$ and CP in Figure~\ref{fig:conformal-like-coverage_supp} should be interpreted as a sanity check rather than a failure case.

\begin{figure}[!tb]
    \centering
    \includegraphics[width=0.95\linewidth]{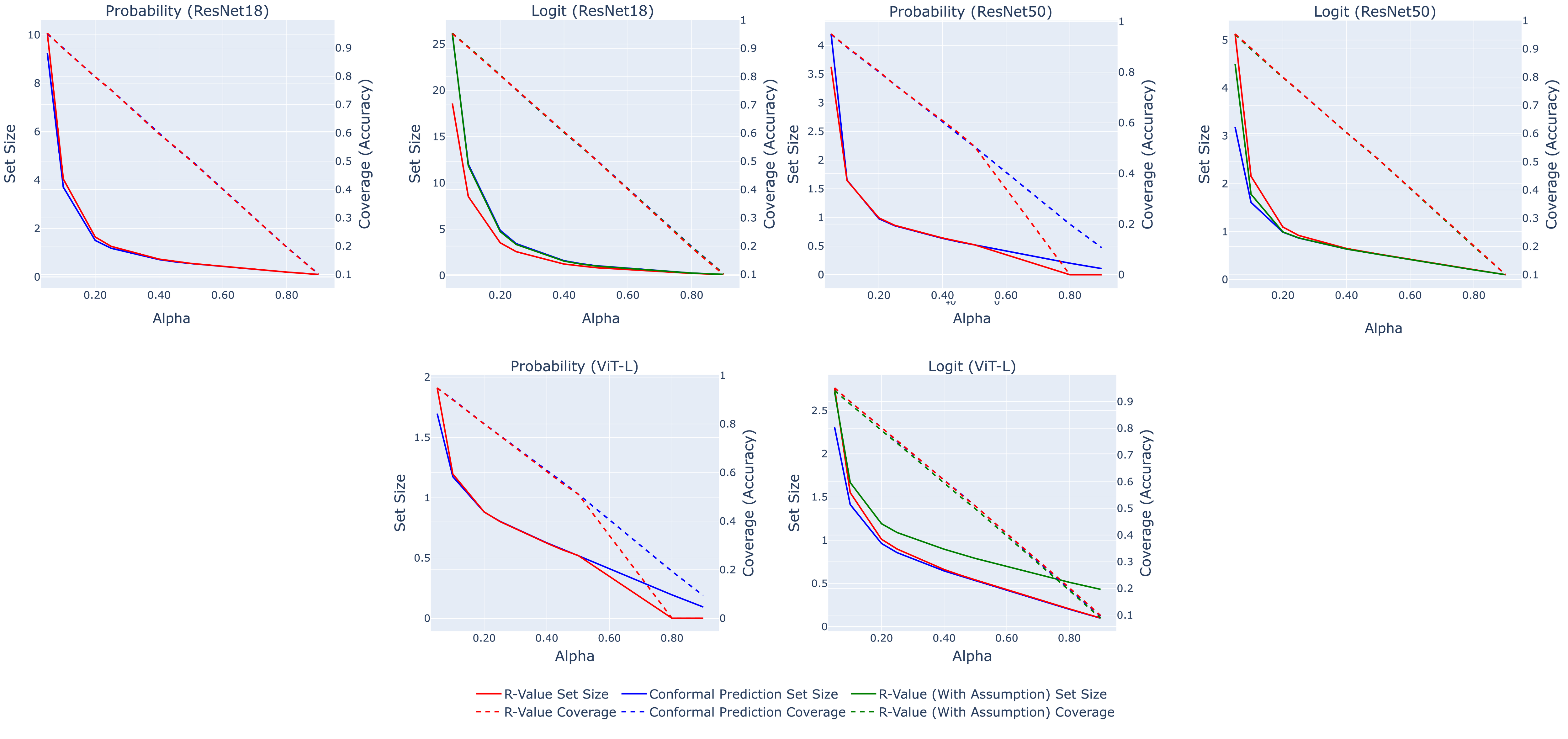}
    \caption{Comparison of conformal prediction and $\CPr$ on image classification using ResNet18, ResNet50, ViT-Base, and ViT-Large. We evaluate both probability and logit settings, analyzing coverage and set size across significance levels. In the single-backbone setting, WBB adapters introduce only modest posterior variability, so $\CPr$ often behaves similarly to standard CP.}
    \label{fig:conformal-like-coverage_supp}
\end{figure}

\textbf{Accuracy trade-off from adapter training.}
Using adapters provides an efficient way to estimate model variability, but it also introduces a accuracy trade off. Since only a small fraction of the model parameters are trainable while the pretrained backbone is frozen, the adapter augmented models may not fully preserve the accuracy of the original backbone. Table~\ref{tab:adapter_acc} reports this effect. The degradation is modest across all models. For ViT models, the drop is around $1$-$1.5\%$, while for ResNet50, it is much smaller. This accuracy loss should be viewed as the cost of obtaining an efficient approximation of epistemic variability. It is not an inherent limitation of the $r$-value framework itself.

\begin{table}[!tb]
    \centering
    \begin{tabular}{lcc}
        \toprule
        Model & Without Adapters & With Adapters (Average) \\
        \midrule
        ViT-L/16 & 85.15\% & 83.73\% \\
        ViT-B/16 & 81.16\% & 79.87\% \\
        ResNet50 & 79.26\% & 79.05\% \\
        ResNet18 & 71.12\% & 69.95\% \\
        \bottomrule
    \end{tabular}
    
    \caption{Raw classification accuracy before and after introducing WBB-trained adapters. The adapter-based posterior samples provide an efficient estimate of epistemic variability, but the lightweight parameterization can introduce a modest accuracy drop.}
    \label{tab:adapter_acc}
\end{table}

\textbf{Effect of variability quality.}
The effectiveness of $\CPr$ depends on the quality of the estimated variability. In our current implementation, we use adapters because they make WBB posterior sampling computationally feasible for large vision models. This choice is efficient, but the resulting posterior samples may only capture part of the full model variability. As a result, the variance signal can be weak in the single backbone setting, which makes $\CPr$ behave similarly to CP. This is consistent with our theoretical analysis. When variability vanishes or is uninformative, the $r$-value reduces to a first order ranking. When variability is informative, $\CPr$ can use it to penalize unstable false labels and reduce the conformal set size.

This also suggests that better variability estimation can further improve the empirical performance of $\CPr$. Our method is not tied to shallow adapters. Other efficient posterior-sampling or uncertainty-estimation methods can be used as long as they provide meaningful posterior score samples. For example, MC-dropout, Deep Ensembles, or better calibrated paraphrase sampling in VLM and LLM settings may provide higher quality variability estimates. In such cases, the $r$-value ranking may better distinguish stable high score labels from labels that only appear high due to noise. (\textbf{More experimental comparisons will be discussed in the next subsection.} )

\textbf{Ranking quality within the conformal set.}
Beyond coverage and set size, we also examine where the true label appears inside the selected conformal set. Table~\ref{tab:index_avg} reports the average index of the true label. A smaller value means that the true label appears earlier in the ranked set. For ResNet18, ResNet50, and ViT-Base, the $r$-value rankings place the true label earlier than standard CP on average. This suggests that even when aggregate coverage and set size are similar, $\CPr$ can still improve the ordering of plausible labels. ViT-Large is an exception. We suggest that it is because ViT-Large is trained to overfit; the attached adapters might break this trend. In this case, all methods produce very small sets, and CP has the smallest average index. This is consistent with the high accuracy and low variability regime, where there is less room for ranking improvement and $\CPr$ naturally approaches CP-like behavior.

\begin{table*}[ht]
    \centering
    \begin{tabular}{lcccc}
        \toprule
        Method & \makecell{ResNet18} & \makecell{ResNet50} & \makecell{ViT-B} & \makecell{ViT-L} \\ 
        \midrule

        Conformal Prediction & 1.159 & 0.497 & 0.294 & \bf{0.158} \\
        \bf{$r$-value Under Logits} & 1.092 & 0.351 & \bf{0.272} & 0.197 \\
        \bf{$r$-value Under Logits (with Assumption)} & \bf{1.082} & \bf{0.348} & 0.278 & 0.164 \\
        
        \bottomrule
   \end{tabular}
    \caption{Average index of the true label within the selected conformal set. Smaller values indicate that the true label appears earlier in the ranked set.}
    \label{tab:index_avg}
\end{table*}

%% file: AISTATS/sections/Appendix/full_tables.tex
\paragraph{Effect of paraphrase quality.}
The quality of the variability source also affects the benefit of $\CPr$.
To study this, we compare using all 30 randomly generated paraphrases with using the top 20 rephrased prompts selected from them.
Both settings maintain coverage close to the nominal level, but the Top-20 setting often gives slightly smaller conformal sets or larger relative reductions, especially for MobileCLIP2 variants.
For example, under MobileCLIP2-S2, the gain over $\CPavg$ increases from $1.4\%$ with 30 paraphrases to $5.8\%$ with Top-20 paraphrases, and the gain over $\CP$ increases from $3.8\%$ to $7.6\%$.
Similarly, MobileCLIP2-S4 improves from $4.9\%$ to $6.2\%$ over $\CPavg$ and from $5.2\%$ to $6.6\%$ over $\CP$.

This supports the intuition that $\CPr$ benefits not only from having multiple perturbations, but also from having informative perturbations.
Low quality or noisy paraphrases may introduce variability that is less aligned with meaningful epistemic uncertainty, weakening the rank stability signal.
In contrast, higher quality paraphrases provide a cleaner estimate of which candidates remain consistently strong across prompt variations, allowing $\CPr$ to better penalize unstable false candidates while preserving coverage.

\begin{table*}[h]
  \centering
  \caption{Model-level macro-average coverage and set size across datasets at $\alpha=0.05$ (30 paraphrases). Set-size parentheses report the relative change of $r$-value versus the corresponding CP baseline, computed from the unrounded set-size means.}
  \label{tab:main_regular30}
  \scriptsize
  \setlength{\tabcolsep}{2pt}
  \renewcommand{\arraystretch}{1.06}
  \begin{tabular}{@{}lcccccc@{}}
    \toprule
    & \multicolumn{3}{c}{Coverage (\%)} & \multicolumn{3}{c}{Set size} \\
    \cmidrule(lr){2-4} \cmidrule(lr){5-7}
    Model & $\CPr$ & $\CPavg$ & $\CP$ & $\CPr$ & $\CPavg$ & $\CP$ \\
    \midrule
    SigLIP2-B/16 & $\mathbf{95.2 \pm 1.5}$ & $95.4 \pm 1.4$ & $95.5 \pm 1.4$ & $\mathbf{4.0 \pm 0.5}$ & \cellrel{$4.6 \pm 0.6$}{-13.1} & \cellrel{$5.3 \pm 0.9$}{-23.9} \\
    SigLIP2-B/32 & $\mathbf{95.1 \pm 1.6}$ & $95.3 \pm 1.5$ & $95.4 \pm 1.4$ & $\mathbf{11.8 \pm 1.1}$ & \cellrel{$13.1 \pm 1.4$}{-10.3} & \cellrel{$14.1 \pm 1.5$}{-16.8} \\
    SigLIP2-L/16 & $\mathbf{95.2 \pm 1.5}$ & $95.4 \pm 1.3$ & $95.3 \pm 1.4$ & $\mathbf{3.8 \pm 0.4}$ & \cellrel{$4.2 \pm 0.5$}{-9.6} & \cellrel{$4.5 \pm 0.6$}{-15.9} \\
    SigLIP2-SO400M/16 & $\mathbf{95.0 \pm 1.5}$ & $95.6 \pm 1.3$ & $95.5 \pm 1.4$ & $\mathbf{3.3 \pm 0.3}$ & \cellrel{$3.7 \pm 0.5$}{-11.0} & \cellrel{$4.0 \pm 0.5$}{-16.4} \\
    MobileCLIP2-S2 & $\mathbf{95.2 \pm 1.5}$ & $95.7 \pm 1.3$ & $95.7 \pm 1.2$ & $\mathbf{5.9 \pm 0.5}$ & \cellrel{$6.0 \pm 0.6$}{-1.4} & \cellrel{$6.1 \pm 0.6$}{-3.8} \\
    MobileCLIP2-S3 & $\mathbf{95.3 \pm 1.5}$ & $95.7 \pm 1.3$ & $95.6 \pm 1.3$ & $\mathbf{4.0 \pm 0.3}$ & \cellrel{$4.2 \pm 0.4$}{-6.3} & \cellrel{$4.4 \pm 0.3$}{-10.3} \\
    MobileCLIP2-S4 & $\mathbf{95.3 \pm 1.5}$ & $95.6 \pm 1.3$ & $95.6 \pm 1.3$ & $\mathbf{3.2 \pm 0.3}$ & \cellrel{$3.4 \pm 0.3$}{-4.9} & \cellrel{$3.4 \pm 0.4$}{-5.2} \\
    MobileCLIP2-B & $\mathbf{95.2 \pm 1.5}$ & $95.6 \pm 1.3$ & $95.7 \pm 1.3$ & $\mathbf{4.7 \pm 0.6}$ & \cellrel{$4.8 \pm 0.5$}{-3.4} & \cellrel{$5.2 \pm 0.5$}{-10.0} \\
    \bottomrule
  \end{tabular}
\end{table*}

\begin{table*}[h]
  \centering
  \caption{Model-level macro-average coverage and set size across datasets at $\alpha=0.05$ (Top-20 rephrased). Set-size parentheses report the relative change of $r$-value versus the corresponding CP baseline, computed from the unrounded set-size means.}
  \label{tab:main_top20}
  \scriptsize
  \setlength{\tabcolsep}{2pt}
  \renewcommand{\arraystretch}{1.06}
  \begin{tabular}{@{}lcccccc@{}}
    \toprule
    & \multicolumn{3}{c}{Coverage (\%)} & \multicolumn{3}{c}{Set size} \\
    \cmidrule(lr){2-4} \cmidrule(lr){5-7}
    Model & $\CPr$ & $\CPavg$ & $\CP$ & $\CPr$ & $\CPavg$ & $\CP$ \\
    \midrule
    SigLIP2-B/16 & $\mathbf{95.1 \pm 1.6}$ & $95.4 \pm 1.4$ & $95.5 \pm 1.4$ & $\mathbf{4.0 \pm 0.5}$ & \cellrel{$4.5 \pm 0.6$}{-11.3} & \cellrel{$5.3 \pm 0.9$}{-23.6} \\
    SigLIP2-B/32 & $\mathbf{95.1 \pm 1.5}$ & $95.3 \pm 1.5$ & $95.4 \pm 1.4$ & $\mathbf{11.9 \pm 1.1}$ & \cellrel{$13.2 \pm 1.4$}{-10.0} & \cellrel{$14.1 \pm 1.5$}{-15.6} \\
    SigLIP2-L/16 & $\mathbf{95.0 \pm 1.5}$ & $95.4 \pm 1.4$ & $95.3 \pm 1.4$ & $\mathbf{3.7 \pm 0.4}$ & \cellrel{$4.2 \pm 0.6$}{-11.4} & \cellrel{$4.5 \pm 0.6$}{-16.6} \\
    SigLIP2-SO400M/16 & $\mathbf{94.9 \pm 1.6}$ & $95.6 \pm 1.3$ & $95.5 \pm 1.4$ & $\mathbf{3.3 \pm 0.3}$ & \cellrel{$3.6 \pm 0.4$}{-8.7} & \cellrel{$3.9 \pm 0.5$}{-17.4} \\
    MobileCLIP2-S2 & $\mathbf{95.1 \pm 1.5}$ & $95.7 \pm 1.3$ & $95.7 \pm 1.3$ & $\mathbf{5.6 \pm 0.5}$ & \cellrel{$6.0 \pm 0.6$}{-5.8} & \cellrel{$6.1 \pm 0.6$}{-7.6} \\
    MobileCLIP2-S3 & $\mathbf{95.2 \pm 1.5}$ & $95.7 \pm 1.3$ & $95.7 \pm 1.2$ & $\mathbf{4.0 \pm 0.4}$ & \cellrel{$4.2 \pm 0.3$}{-4.9} & \cellrel{$4.4 \pm 0.3$}{-9.7} \\
    MobileCLIP2-S4 & $\mathbf{95.2 \pm 1.5}$ & $95.6 \pm 1.3$ & $95.6 \pm 1.3$ & $\mathbf{3.2 \pm 0.3}$ & \cellrel{$3.4 \pm 0.3$}{-6.2} & \cellrel{$3.4 \pm 0.4$}{-6.6} \\
    MobileCLIP2-B & $\mathbf{95.2 \pm 1.5}$ & $95.7 \pm 1.4$ & $95.7 \pm 1.3$ & $\mathbf{4.7 \pm 0.6}$ & \cellrel{$4.9 \pm 0.5$}{-4.1} & \cellrel{$5.2 \pm 0.5$}{-10.1} \\
    \bottomrule
  \end{tabular}
\end{table*}

\begingroup
\scriptsize
\setlength{\tabcolsep}{2pt}
\renewcommand{\arraystretch}{1.05}
\begin{longtable}{@{}llcccccc@{}}
\caption{Full coverage and set-size results at $\alpha=0.05$ (30 paraphrases). Set-size parentheses report the relative change of $r$-value versus the corresponding CP baseline, computed from the unrounded set-size means.}\label{tab:app_regular30}\\
\toprule
Model & Dataset & \multicolumn{3}{c}{Coverage (\%)} & \multicolumn{3}{c}{Set size} \\
\cmidrule(lr){3-5} \cmidrule(lr){6-8}
& & $\CPr$ & $\CPavg$ & $\CP$ & $\CPr$ & $\CPavg$ & $\CP$ \\
\midrule
\endfirsthead
\caption[]{Full coverage and set-size results at $\alpha=0.05$ (Regular 30, continued).}\\
\toprule
Model & Dataset & \multicolumn{3}{c}{Coverage (\%)} & \multicolumn{3}{c}{Set size} \\
\cmidrule(lr){3-5} \cmidrule(lr){6-8}
& & $\CPr$ & $\CPavg$ & $\CP$ & $\CPr$ & $\CPavg$ & $\CP$ \\
\midrule
\endhead
\midrule
\multicolumn{8}{r}{\emph{Continued on next page}}\\
\endfoot
\bottomrule
\endlastfoot
 \multirow{6}{*}{CLIP-B/16} & CIFAR-10 & $\mathbf{95.5 \pm 2.0}$ & $95.6 \pm 1.3$ & $95.5 \pm 1.6$ & $\mathbf{1.3 \pm 0.1}$ & \cellrel{$1.4 \pm 0.1$}{-6.6} & \cellrel{$1.4 \pm 0.1$}{-8.3} \\
  & CIFAR-100 & $\mathbf{95.1 \pm 1.7}$ & $95.8 \pm 1.8$ & $95.4 \pm 1.7$ & $\mathbf{7.8 \pm 0.6}$ & \cellrel{$8.2 \pm 1.2$}{-5.5} & \cellrel{$8.3 \pm 1.1$}{-7.1} \\
  & ImageNet & $\mathbf{95.2 \pm 1.4}$ & $95.3 \pm 1.2$ & $95.1 \pm 1.4$ & $\mathbf{8.6 \pm 1.1}$ & \cellrel{$8.9 \pm 1.2$}{-3.1} & \cellrel{$10.2 \pm 2.4$}{-16.0} \\
  & ImageNet-A & $\mathbf{95.1 \pm 1.3}$ & $95.5 \pm 1.4$ & $95.2 \pm 1.4$ & $\mathbf{22.5 \pm 3.3}$ & \cellrel{$23.5 \pm 2.5$}{-4.6} & \cellrel{$24.7 \pm 2.7$}{-9.0} \\
  & ImageNet-R & $\mathbf{95.0 \pm 1.1}$ & $95.1 \pm 1.3$ & $94.9 \pm 1.4$ & $\mathbf{7.2 \pm 1.4}$ & \cellrel{$7.3 \pm 1.4$}{-2.3} & \cellrel{$8.4 \pm 1.5$}{-14.7} \\
  & EuroSAT & $\mathbf{94.6 \pm 1.4}$ & $95.2 \pm 1.5$ & $95.3 \pm 1.3$ & $\mathbf{5.6 \pm 0.2}$ & \cellrel{$5.9 \pm 0.3$}{-5.6} & \cellrel{$6.0 \pm 0.2$}{-6.4} \\
\midrule
 \multirow{6}{*}{CLIP-B/32} & CIFAR-10 & $\mathbf{95.3 \pm 2.2}$ & $95.8 \pm 1.7$ & $95.8 \pm 1.5$ & $\mathbf{1.2 \pm 0.1}$ & \cellrel{$1.2 \pm 0.1$}{-4.0} & \cellrel{$1.2 \pm 0.1$}{-3.1} \\
  & CIFAR-100 & $\mathbf{95.5 \pm 1.5}$ & $95.4 \pm 2.0$ & $95.7 \pm 1.6$ & $\mathbf{7.3 \pm 0.7}$ & \cellrel{$7.7 \pm 1.0$}{-4.5} & \cellrel{$9.3 \pm 1.8$}{-21.6} \\
  & ImageNet & $\mathbf{95.1 \pm 1.5}$ & $95.3 \pm 1.4$ & $95.0 \pm 1.6$ & $\mathbf{12.7 \pm 2.1}$ & \cellrel{$12.7 \pm 1.5$}{0.0} & \cellrel{$14.1 \pm 2.4$}{-10.4} \\
  & ImageNet-A & $\mathbf{95.2 \pm 1.5}$ & $95.3 \pm 1.4$ & $95.0 \pm 1.2$ & $\mathbf{47.5 \pm 4.8}$ & \cellrel{$47.6 \pm 4.0$}{-0.2} & \cellrel{$60.4 \pm 7.3$}{-21.4} \\
  & ImageNet-R & $\mathbf{95.1 \pm 1.3}$ & $95.0 \pm 1.3$ & $95.1 \pm 1.3$ & $\mathbf{13.6 \pm 3.0}$ & \cellrel{$14.0 \pm 2.2$}{-2.8} & \cellrel{$14.9 \pm 2.4$}{-8.9} \\
  & EuroSAT & $\mathbf{94.9 \pm 1.5}$ & $95.4 \pm 1.3$ & $95.2 \pm 1.4$ & $\mathbf{6.2 \pm 0.4}$ & \cellrel{$6.2 \pm 0.4$}{-0.2} & \cellrel{$7.0 \pm 0.3$}{-11.9} \\
\midrule
 \multirow{6}{*}{SigLIP2-B/16} & CIFAR-10 & $\mathbf{95.6 \pm 1.7}$ & $95.9 \pm 1.6$ & $96.3 \pm 1.5$ & $\mathbf{1.0 \pm 0.0}$ & \cellrel{$1.1 \pm 0.0$}{-2.5} & \cellrel{$1.0 \pm 0.0$}{-0.1} \\
  & CIFAR-100 & $\mathbf{95.3 \pm 2.0}$ & $95.5 \pm 2.0$ & $95.6 \pm 1.7$ & $\mathbf{4.6 \pm 0.8}$ & \cellrel{$5.0 \pm 0.8$}{-7.6} & \cellrel{$5.0 \pm 1.0$}{-7.8} \\
  & ImageNet & $\mathbf{94.9 \pm 1.3}$ & $95.0 \pm 1.1$ & $95.4 \pm 1.4$ & $\mathbf{7.0 \pm 1.5}$ & \cellrel{$8.6 \pm 1.3$}{-18.4} & \cellrel{$10.7 \pm 2.8$}{-34.8} \\
  & ImageNet-A & $\mathbf{95.2 \pm 1.2}$ & $95.3 \pm 1.2$ & $95.1 \pm 1.4$ & $\mathbf{5.3 \pm 0.6}$ & \cellrel{$5.9 \pm 0.6$}{-10.4} & \cellrel{$6.9 \pm 1.0$}{-23.3} \\
  & ImageNet-R & $\mathbf{95.3 \pm 1.6}$ & $95.4 \pm 1.3$ & $95.3 \pm 1.2$ & $\mathbf{1.1 \pm 0.0}$ & \cellrel{$2.2 \pm 0.6$}{-47.7} & \cellrel{$1.8 \pm 0.3$}{-36.5} \\
  & EuroSAT & $\mathbf{94.9 \pm 1.4}$ & $95.2 \pm 1.4$ & $95.3 \pm 1.1$ & $\mathbf{5.1 \pm 0.2}$ & \cellrel{$5.1 \pm 0.1$}{-0.5} & \cellrel{$6.2 \pm 0.2$}{-19.1} \\
\midrule
 \multirow{6}{*}{SigLIP2-B/32} & CIFAR-10 & $\mathbf{95.4 \pm 1.7}$ & $95.6 \pm 1.8$ & $96.0 \pm 1.8$ & $\mathbf{1.1 \pm 0.0}$ & \cellrel{$1.1 \pm 0.0$}{-0.9} & \cellrel{$1.1 \pm 0.0$}{-2.7} \\
  & CIFAR-100 & $\mathbf{95.3 \pm 2.0}$ & $95.6 \pm 1.7$ & $95.7 \pm 1.5$ & $\mathbf{5.1 \pm 0.8}$ & \cellrel{$5.8 \pm 1.2$}{-13.0} & \cellrel{$5.4 \pm 1.0$}{-5.6} \\
  & ImageNet & $\mathbf{94.8 \pm 1.3}$ & $95.1 \pm 1.4$ & $95.1 \pm 1.2$ & $\mathbf{13.0 \pm 2.1}$ & \cellrel{$14.3 \pm 2.3$}{-9.3} & \cellrel{$19.6 \pm 2.8$}{-33.6} \\
  & ImageNet-A & $\mathbf{95.0 \pm 1.5}$ & $94.9 \pm 1.4$ & $95.0 \pm 1.4$ & $\mathbf{43.5 \pm 3.4}$ & \cellrel{$46.2 \pm 3.3$}{-6.0} & \cellrel{$48.6 \pm 4.7$}{-10.6} \\
  & ImageNet-R & $\mathbf{95.0 \pm 1.4}$ & $94.9 \pm 1.3$ & $95.3 \pm 1.2$ & $\mathbf{2.2 \pm 0.3}$ & \cellrel{$5.4 \pm 1.2$}{-59.6} & \cellrel{$3.6 \pm 0.5$}{-38.2} \\
  & EuroSAT & $\mathbf{95.1 \pm 1.4}$ & $95.3 \pm 1.3$ & $95.3 \pm 1.3$ & $\mathbf{5.8 \pm 0.2}$ & \cellrel{$5.8 \pm 0.3$}{-0.5} & \cellrel{$6.6 \pm 0.2$}{-12.4} \\
\midrule
 \multirow{6}{*}{SigLIP2-L/16} & CIFAR-10 & $\mathbf{95.7 \pm 1.7}$ & $96.0 \pm 1.2$ & $95.5 \pm 1.7$ & $\mathbf{1.1 \pm 0.0}$ & \cellrel{$1.1 \pm 0.0$}{-1.5} & \cellrel{$1.1 \pm 0.0$}{-3.2} \\
  & CIFAR-100 & $\mathbf{95.3 \pm 1.8}$ & $95.5 \pm 1.8$ & $95.4 \pm 1.8$ & $\mathbf{4.7 \pm 0.5}$ & \cellrel{$6.3 \pm 1.1$}{-25.9} & \cellrel{$5.0 \pm 0.8$}{-6.3} \\
  & ImageNet & $\mathbf{94.9 \pm 1.3}$ & $95.3 \pm 1.3$ & $95.0 \pm 1.3$ & $\mathbf{9.2 \pm 1.6}$ & \cellrel{$9.2 \pm 1.5$}{-0.6} & \cellrel{$12.3 \pm 2.3$}{-25.3} \\
  & ImageNet-A & $\mathbf{95.0 \pm 1.3}$ & $95.3 \pm 1.2$ & $95.3 \pm 1.3$ & $\mathbf{2.2 \pm 0.2}$ & \cellrel{$2.3 \pm 0.2$}{-6.8} & \cellrel{$2.5 \pm 0.4$}{-14.4} \\
  & ImageNet-R & $\mathbf{95.1 \pm 1.3}$ & $95.4 \pm 1.0$ & $95.4 \pm 0.9$ & $\mathbf{1.0 \pm 0.0}$ & \cellrel{$1.2 \pm 0.1$}{-14.0} & \cellrel{$1.2 \pm 0.1$}{-12.0} \\
  & EuroSAT & $\mathbf{95.0 \pm 1.4}$ & $94.9 \pm 1.4$ & $95.2 \pm 1.3$ & $\mathbf{4.7 \pm 0.2}$ & \cellrel{$5.1 \pm 0.1$}{-7.7} & \cellrel{$5.1 \pm 0.2$}{-7.0} \\
\midrule
 \multirow{6}{*}{SigLIP2-SO400M/16} & CIFAR-10 & $\mathbf{95.2 \pm 1.9}$ & $97.0 \pm 0.9$ & $96.0 \pm 1.3$ & $\mathbf{1.0 \pm 0.0}$ & \cellrel{$1.0 \pm 0.0$}{-3.3} & \cellrel{$1.0 \pm 0.0$}{-3.9} \\
  & CIFAR-100 & $\mathbf{95.3 \pm 1.8}$ & $95.4 \pm 2.0$ & $95.3 \pm 2.1$ & $\mathbf{2.6 \pm 0.4}$ & \cellrel{$3.6 \pm 0.6$}{-28.6} & \cellrel{$3.0 \pm 0.5$}{-14.7} \\
  & ImageNet & $\mathbf{94.9 \pm 1.3}$ & $95.3 \pm 1.4$ & $95.3 \pm 1.3$ & $\mathbf{8.8 \pm 1.1}$ & \cellrel{$9.4 \pm 1.7$}{-6.6} & \cellrel{$10.9 \pm 2.1$}{-20.1} \\
  & ImageNet-A & $\mathbf{94.9 \pm 1.3}$ & $95.0 \pm 1.4$ & $95.4 \pm 1.2$ & $\mathbf{2.0 \pm 0.2}$ & \cellrel{$2.0 \pm 0.3$}{-1.6} & \cellrel{$2.1 \pm 0.3$}{-7.5} \\
  & ImageNet-R & $\mathbf{94.8 \pm 1.4}$ & $95.4 \pm 1.1$ & $95.7 \pm 1.0$ & $\mathbf{1.0 \pm 0.0}$ & \cellrel{$1.3 \pm 0.1$}{-19.6} & \cellrel{$1.2 \pm 0.2$}{-18.9} \\
  & EuroSAT & $\mathbf{94.6 \pm 1.4}$ & $95.4 \pm 1.2$ & $95.3 \pm 1.2$ & $\mathbf{4.6 \pm 0.2}$ & \cellrel{$5.1 \pm 0.2$}{-10.0} & \cellrel{$5.4 \pm 0.2$}{-15.2} \\
\midrule
 \multirow{6}{*}{MobileCLIP2-S2} & CIFAR-10 & $\mathbf{95.1 \pm 2.2}$ & $97.6 \pm 0.7$ & $97.8 \pm 0.6$ & $\mathbf{1.0 \pm 0.0}$ & \cellrel{$1.0 \pm 0.0$}{-3.8} & \cellrel{$1.0 \pm 0.0$}{-3.8} \\
  & CIFAR-100 & $\mathbf{95.4 \pm 1.9}$ & $95.7 \pm 1.8$ & $95.6 \pm 1.9$ & $\mathbf{1.6 \pm 0.1}$ & \cellrel{$1.7 \pm 0.1$}{-4.4} & \cellrel{$1.6 \pm 0.2$}{-2.6} \\
  & ImageNet & $\mathbf{95.1 \pm 1.2}$ & $95.1 \pm 1.3$ & $95.3 \pm 1.4$ & $\mathbf{4.1 \pm 0.6}$ & \cellrel{$4.3 \pm 0.6$}{-2.9} & \cellrel{$4.4 \pm 0.7$}{-5.0} \\
  & ImageNet-A & $\mathbf{95.3 \pm 1.3}$ & $95.4 \pm 1.3$ & $95.0 \pm 1.2$ & $\mathbf{21.9 \pm 2.0}$ & \cellrel{$22.1 \pm 2.3$}{-0.8} & \cellrel{$22.0 \pm 2.5$}{-0.6} \\
  & ImageNet-R & $\mathbf{95.3 \pm 1.3}$ & $95.3 \pm 1.4$ & $95.1 \pm 1.1$ & $\mathbf{1.9 \pm 0.2}$ & \cellrel{$1.9 \pm 0.3$}{-2.1} & \cellrel{$1.9 \pm 0.2$}{-1.6} \\
  & EuroSAT & $\mathbf{95.1 \pm 1.3}$ & $95.2 \pm 1.4$ & $95.3 \pm 1.2$ & $\mathbf{4.8 \pm 0.2}$ & \cellrel{$4.8 \pm 0.2$}{-1.4} & \cellrel{$5.7 \pm 0.2$}{-16.4} \\
\midrule
 \multirow{6}{*}{MobileCLIP2-S3} & CIFAR-10 & $\mathbf{95.7 \pm 1.9}$ & $97.6 \pm 0.7$ & $97.6 \pm 0.7$ & $\mathbf{1.0 \pm 0.0}$ & \cellrel{$1.0 \pm 0.0$}{-3.1} & \cellrel{$1.0 \pm 0.0$}{-3.1} \\
  & CIFAR-100 & $\mathbf{95.4 \pm 1.8}$ & $95.5 \pm 1.8$ & $95.7 \pm 1.8$ & $\mathbf{1.6 \pm 0.2}$ & \cellrel{$1.6 \pm 0.2$}{-0.1} & \cellrel{$1.6 \pm 0.2$}{-0.4} \\
  & ImageNet & $\mathbf{95.1 \pm 1.2}$ & $95.4 \pm 1.4$ & $95.2 \pm 1.3$ & $\mathbf{3.2 \pm 0.4}$ & \cellrel{$3.3 \pm 0.4$}{-2.7} & \cellrel{$3.3 \pm 0.4$}{-4.4} \\
  & ImageNet-A & $\mathbf{95.0 \pm 1.4}$ & $95.1 \pm 1.3$ & $95.0 \pm 1.2$ & $\mathbf{12.0 \pm 1.1}$ & \cellrel{$13.3 \pm 1.5$}{-9.3} & \cellrel{$13.8 \pm 1.1$}{-12.8} \\
  & ImageNet-R & $\mathbf{95.5 \pm 1.3}$ & $95.4 \pm 1.1$ & $95.0 \pm 1.4$ & $\mathbf{1.3 \pm 0.1}$ & \cellrel{$1.4 \pm 0.1$}{-3.2} & \cellrel{$1.4 \pm 0.1$}{-2.4} \\
  & EuroSAT & $\mathbf{95.1 \pm 1.2}$ & $95.2 \pm 1.3$ & $95.2 \pm 1.1$ & $\mathbf{4.6 \pm 0.2}$ & \cellrel{$4.8 \pm 0.2$}{-4.0} & \cellrel{$5.3 \pm 0.2$}{-13.7} \\
\midrule
 \multirow{6}{*}{MobileCLIP2-S4} & CIFAR-10 & $\mathbf{95.6 \pm 1.8}$ & $97.4 \pm 0.7$ & $97.7 \pm 0.7$ & $\mathbf{1.0 \pm 0.0}$ & \cellrel{$1.0 \pm 0.0$}{-3.6} & \cellrel{$1.0 \pm 0.0$}{-3.6} \\
  & CIFAR-100 & $\mathbf{95.5 \pm 2.0}$ & $95.7 \pm 1.8$ & $95.5 \pm 1.7$ & $\mathbf{1.3 \pm 0.1}$ & \cellrel{$1.3 \pm 0.1$}{-1.9} & \cellrel{$1.3 \pm 0.1$}{-0.4} \\
  & ImageNet & $\mathbf{95.1 \pm 1.1}$ & $95.0 \pm 1.4$ & $95.1 \pm 1.4$ & $\mathbf{3.1 \pm 0.5}$ & \cellrel{$3.1 \pm 0.3$}{-0.7} & \cellrel{$3.1 \pm 0.7$}{-2.3} \\
  & ImageNet-A & $\mathbf{95.1 \pm 1.3}$ & $95.1 \pm 1.3$ & $94.9 \pm 1.2$ & $\mathbf{8.7 \pm 0.8}$ & \cellrel{$9.4 \pm 1.2$}{-7.4} & \cellrel{$9.2 \pm 1.2$}{-4.7} \\
  & ImageNet-R & $\mathbf{95.2 \pm 1.4}$ & $95.1 \pm 1.2$ & $95.2 \pm 1.3$ & $\mathbf{1.2 \pm 0.1}$ & \cellrel{$1.3 \pm 0.1$}{-1.6} & \cellrel{$1.3 \pm 0.1$}{-1.4} \\
  & EuroSAT & $\mathbf{95.0 \pm 1.5}$ & $95.2 \pm 1.2$ & $95.1 \pm 1.5$ & $\mathbf{4.0 \pm 0.1}$ & \cellrel{$4.3 \pm 0.2$}{-4.8} & \cellrel{$4.5 \pm 0.3$}{-10.9} \\
\midrule
 \multirow{6}{*}{MobileCLIP2-B} & CIFAR-10 & $\mathbf{95.6 \pm 1.9}$ & $97.5 \pm 0.7$ & $97.1 \pm 0.7$ & $\mathbf{1.0 \pm 0.0}$ & \cellrel{$1.0 \pm 0.0$}{-3.5} & \cellrel{$1.0 \pm 0.0$}{-3.5} \\
  & CIFAR-100 & $\mathbf{95.6 \pm 1.8}$ & $95.6 \pm 2.0$ & $96.0 \pm 1.7$ & $\mathbf{1.5 \pm 0.1}$ & \cellrel{$1.5 \pm 0.1$}{-0.2} & \cellrel{$1.6 \pm 0.2$}{-2.0} \\
  & ImageNet & $\mathbf{95.3 \pm 1.1}$ & $95.3 \pm 1.3$ & $95.2 \pm 1.3$ & $\mathbf{3.5 \pm 0.5}$ & \cellrel{$3.5 \pm 0.7$}{-0.6} & \cellrel{$3.6 \pm 0.8$}{-3.4} \\
  & ImageNet-A & $\mathbf{95.2 \pm 1.5}$ & $95.0 \pm 1.5$ & $95.3 \pm 1.3$ & $\mathbf{16.5 \pm 2.7}$ & \cellrel{$17.0 \pm 1.6$}{-3.2} & \cellrel{$18.6 \pm 1.6$}{-11.3} \\
  & ImageNet-R & $\mathbf{94.6 \pm 1.4}$ & $95.0 \pm 1.2$ & $95.2 \pm 1.3$ & $\mathbf{1.5 \pm 0.1}$ & \cellrel{$1.5 \pm 0.1$}{-6.1} & \cellrel{$1.5 \pm 0.1$}{-1.9} \\
  & EuroSAT & $\mathbf{95.1 \pm 1.4}$ & $95.0 \pm 1.3$ & $95.5 \pm 1.4$ & $\mathbf{4.1 \pm 0.2}$ & \cellrel{$4.4 \pm 0.2$}{-6.6} & \cellrel{$4.9 \pm 0.2$}{-16.4} \\
\end{longtable}
\endgroup

\begingroup
\scriptsize
\setlength{\tabcolsep}{2pt}
\renewcommand{\arraystretch}{1.05}
\begin{longtable}{@{}llcccccc@{}}
\caption{Full coverage and set-size results at $\alpha=0.05$ (Top-20 rephrased). Set-size parentheses report the relative change of $r$-value versus the corresponding CP baseline, computed from the unrounded set-size means.}\label{tab:app_top20}\\
\toprule
Model & Dataset & \multicolumn{3}{c}{Coverage (\%)} & \multicolumn{3}{c}{Set size} \\
\cmidrule(lr){3-5} \cmidrule(lr){6-8}
& & $\CPr$ & $\CPavg$ & $\CP$ & $\CPr$ & $\CPavg$ & $\CP$ \\
\midrule
\endfirsthead
\caption[]{Full coverage and set-size results at $\alpha=0.05$ (Top-20 rephrased, continued).}\\
\toprule
Model & Dataset & \multicolumn{3}{c}{Coverage (\%)} & \multicolumn{3}{c}{Set size} \\
\cmidrule(lr){3-5} \cmidrule(lr){6-8}
& & $\CPr$ & $\CPavg$ & $\CP$ & $\CPr$ & $\CPavg$ & $\CP$ \\
\midrule
\endhead
\midrule
\multicolumn{8}{r}{\emph{Continued on next page}}\\
\endfoot
\bottomrule
\endlastfoot
 \multirow{6}{*}{CLIP-B/16} & CIFAR-10 & $\mathbf{95.6 \pm 1.9}$ & $95.6 \pm 1.4$ & $95.5 \pm 1.6$ & $\mathbf{1.2 \pm 0.1}$ & \cellrel{$1.3 \pm 0.1$}{-3.9} & \cellrel{$1.4 \pm 0.1$}{-10.2} \\
  & CIFAR-100 & $\mathbf{95.0 \pm 1.8}$ & $95.8 \pm 1.8$ & $95.4 \pm 1.7$ & $\mathbf{7.6 \pm 0.7}$ & \cellrel{$8.2 \pm 0.9$}{-7.4} & \cellrel{$8.3 \pm 1.1$}{-8.5} \\
  & ImageNet & $\mathbf{95.1 \pm 1.5}$ & $95.3 \pm 1.2$ & $95.1 \pm 1.4$ & $\mathbf{8.5 \pm 1.2}$ & \cellrel{$8.7 \pm 1.2$}{-1.5} & \cellrel{$10.2 \pm 2.4$}{-16.5} \\
  & ImageNet-A & $\mathbf{95.1 \pm 1.3}$ & $95.5 \pm 1.4$ & $95.2 \pm 1.4$ & $\mathbf{21.4 \pm 2.4}$ & \cellrel{$21.6 \pm 2.8$}{-0.8} & \cellrel{$24.7 \pm 2.7$}{-13.3} \\
  & ImageNet-R & $\mathbf{94.8 \pm 1.4}$ & $95.2 \pm 1.2$ & $95.2 \pm 1.2$ & $\mathbf{7.2 \pm 1.5}$ & \cellrel{$7.6 \pm 1.5$}{-6.5} & \cellrel{$8.8 \pm 1.6$}{-18.6} \\
  & EuroSAT & $\mathbf{94.4 \pm 1.6}$ & $95.2 \pm 1.4$ & $95.3 \pm 1.3$ & $\mathbf{5.5 \pm 0.3}$ & \cellrel{$5.7 \pm 0.2$}{-3.4} & \cellrel{$6.0 \pm 0.2$}{-7.7} \\
\midrule
 \multirow{6}{*}{CLIP-B/32} & CIFAR-10 & $\mathbf{95.2 \pm 2.2}$ & $95.8 \pm 1.6$ & $95.8 \pm 1.5$ & $\mathbf{1.2 \pm 0.1}$ & \cellrel{$1.2 \pm 0.1$}{-3.1} & \cellrel{$1.2 \pm 0.1$}{-4.1} \\
  & CIFAR-100 & $\mathbf{95.5 \pm 1.6}$ & $95.3 \pm 2.1$ & $95.7 \pm 1.6$ & $\mathbf{7.4 \pm 0.8}$ & \cellrel{$7.7 \pm 1.2$}{-4.4} & \cellrel{$9.3 \pm 1.8$}{-20.9} \\
  & ImageNet & $\mathbf{95.0 \pm 1.5}$ & $95.3 \pm 1.2$ & $95.0 \pm 1.6$ & $\mathbf{12.5 \pm 1.8}$ & \cellrel{$12.6 \pm 1.5$}{-0.3} & \cellrel{$14.1 \pm 2.4$}{-11.2} \\
  & ImageNet-A & $\mathbf{95.2 \pm 1.4}$ & $95.2 \pm 1.4$ & $95.0 \pm 1.2$ & $\mathbf{47.2 \pm 3.8}$ & \cellrel{$49.4 \pm 4.2$}{-4.5} & \cellrel{$60.4 \pm 7.3$}{-21.9} \\
  & ImageNet-R & $\mathbf{94.7 \pm 1.4}$ & $95.1 \pm 1.3$ & $95.3 \pm 1.3$ & $\mathbf{12.6 \pm 1.8}$ & \cellrel{$14.3 \pm 2.7$}{-11.8} & \cellrel{$15.3 \pm 2.5$}{-17.3} \\
  & EuroSAT & $\mathbf{94.7 \pm 1.6}$ & $95.5 \pm 1.2$ & $95.2 \pm 1.4$ & $\mathbf{5.9 \pm 0.4}$ & \cellrel{$6.2 \pm 0.3$}{-4.7} & \cellrel{$7.0 \pm 0.3$}{-16.1} \\
\midrule
 \multirow{6}{*}{SigLIP2-B/16} & CIFAR-10 & $\mathbf{95.5 \pm 1.7}$ & $95.8 \pm 1.6$ & $96.3 \pm 1.5$ & $\mathbf{1.0 \pm 0.0}$ & \cellrel{$1.1 \pm 0.0$}{-1.9} & \cellrel{$1.0 \pm 0.0$}{-1.0} \\
  & CIFAR-100 & $\mathbf{95.4 \pm 2.0}$ & $95.6 \pm 2.0$ & $95.6 \pm 1.7$ & $\mathbf{4.7 \pm 0.8}$ & \cellrel{$5.2 \pm 1.0$}{-8.4} & \cellrel{$5.0 \pm 1.0$}{-5.7} \\
  & ImageNet & $\mathbf{94.9 \pm 1.3}$ & $95.0 \pm 1.2$ & $95.4 \pm 1.4$ & $\mathbf{7.3 \pm 1.6}$ & \cellrel{$8.3 \pm 1.3$}{-13.0} & \cellrel{$10.7 \pm 2.8$}{-32.3} \\
  & ImageNet-A & $\mathbf{95.0 \pm 1.3}$ & $95.3 \pm 1.2$ & $95.1 \pm 1.4$ & $\mathbf{5.2 \pm 0.6}$ & \cellrel{$5.6 \pm 0.5$}{-5.9} & \cellrel{$6.9 \pm 1.0$}{-23.7} \\
  & ImageNet-R & $\mathbf{95.2 \pm 1.6}$ & $95.4 \pm 1.2$ & $95.3 \pm 1.2$ & $\mathbf{1.1 \pm 0.0}$ & \cellrel{$2.1 \pm 0.5$}{-46.2} & \cellrel{$1.7 \pm 0.2$}{-33.7} \\
  & EuroSAT & $\mathbf{94.4 \pm 1.5}$ & $95.3 \pm 1.2$ & $95.3 \pm 1.1$ & $\mathbf{4.8 \pm 0.2}$ & \cellrel{$5.0 \pm 0.2$}{-4.5} & \cellrel{$6.2 \pm 0.2$}{-23.9} \\
\midrule
 \multirow{6}{*}{SigLIP2-B/32} & CIFAR-10 & $\mathbf{95.5 \pm 1.7}$ & $95.7 \pm 1.7$ & $96.0 \pm 1.8$ & $\mathbf{1.1 \pm 0.0}$ & \cellrel{$1.1 \pm 0.0$}{-0.9} & \cellrel{$1.1 \pm 0.0$}{-4.4} \\
  & CIFAR-100 & $\mathbf{95.4 \pm 2.0}$ & $95.6 \pm 1.7$ & $95.7 \pm 1.5$ & $\mathbf{5.1 \pm 0.8}$ & \cellrel{$6.1 \pm 1.2$}{-17.3} & \cellrel{$5.4 \pm 1.0$}{-6.6} \\
  & ImageNet & $\mathbf{94.9 \pm 1.3}$ & $95.1 \pm 1.4$ & $95.1 \pm 1.2$ & $\mathbf{13.0 \pm 1.6}$ & \cellrel{$14.4 \pm 2.6$}{-9.9} & \cellrel{$19.6 \pm 2.8$}{-33.5} \\
  & ImageNet-A & $\mathbf{94.9 \pm 1.4}$ & $94.9 \pm 1.4$ & $95.0 \pm 1.4$ & $\mathbf{44.6 \pm 3.6}$ & \cellrel{$46.8 \pm 3.1$}{-4.7} & \cellrel{$48.6 \pm 4.7$}{-8.3} \\
  & ImageNet-R & $\mathbf{95.0 \pm 1.4}$ & $94.8 \pm 1.4$ & $95.3 \pm 1.2$ & $\mathbf{2.2 \pm 0.3}$ & \cellrel{$5.2 \pm 1.2$}{-57.8} & \cellrel{$3.3 \pm 0.5$}{-32.6} \\
  & EuroSAT & $\mathbf{94.9 \pm 1.3}$ & $95.4 \pm 1.3$ & $95.3 \pm 1.3$ & $\mathbf{5.5 \pm 0.2}$ & \cellrel{$5.7 \pm 0.2$}{-3.7} & \cellrel{$6.6 \pm 0.2$}{-17.5} \\
\midrule
 \multirow{6}{*}{SigLIP2-L/16} & CIFAR-10 & $\mathbf{95.6 \pm 1.7}$ & $95.9 \pm 1.5$ & $95.5 \pm 1.7$ & $\mathbf{1.1 \pm 0.0}$ & \cellrel{$1.1 \pm 0.0$}{-1.0} & \cellrel{$1.1 \pm 0.0$}{-4.5} \\
  & CIFAR-100 & $\mathbf{95.2 \pm 1.9}$ & $95.5 \pm 1.9$ & $95.4 \pm 1.8$ & $\mathbf{4.7 \pm 0.5}$ & \cellrel{$6.0 \pm 1.2$}{-21.6} & \cellrel{$4.9 \pm 0.8$}{-3.4} \\
  & ImageNet & $\mathbf{94.9 \pm 1.4}$ & $95.3 \pm 1.2$ & $95.0 \pm 1.3$ & $\mathbf{8.8 \pm 1.6}$ & \cellrel{$9.8 \pm 1.9$}{-10.6} & \cellrel{$12.3 \pm 2.3$}{-28.5} \\
  & ImageNet-A & $\mathbf{94.8 \pm 1.4}$ & $95.3 \pm 1.2$ & $95.3 \pm 1.3$ & $\mathbf{2.1 \pm 0.2}$ & \cellrel{$2.2 \pm 0.2$}{-5.7} & \cellrel{$2.5 \pm 0.4$}{-19.3} \\
  & ImageNet-R & $\mathbf{95.0 \pm 1.5}$ & $95.5 \pm 1.0$ & $95.4 \pm 0.9$ & $\mathbf{1.0 \pm 0.0}$ & \cellrel{$1.2 \pm 0.1$}{-13.0} & \cellrel{$1.1 \pm 0.1$}{-11.0} \\
  & EuroSAT & $\mathbf{94.5 \pm 1.4}$ & $94.8 \pm 1.5$ & $95.2 \pm 1.3$ & $\mathbf{4.7 \pm 0.2}$ & \cellrel{$5.0 \pm 0.2$}{-5.1} & \cellrel{$4.8 \pm 0.2$}{-2.1} \\
\midrule
 \multirow{6}{*}{SigLIP2-SO400M/16} & CIFAR-10 & $\mathbf{95.2 \pm 2.0}$ & $96.8 \pm 1.0$ & $96.0 \pm 1.3$ & $\mathbf{1.0 \pm 0.0}$ & \cellrel{$1.0 \pm 0.0$}{-3.0} & \cellrel{$1.0 \pm 0.0$}{-3.9} \\
  & CIFAR-100 & $\mathbf{95.1 \pm 1.9}$ & $95.4 \pm 2.1$ & $95.3 \pm 2.1$ & $\mathbf{2.5 \pm 0.3}$ & \cellrel{$3.3 \pm 0.4$}{-23.9} & \cellrel{$3.0 \pm 0.5$}{-17.4} \\
  & ImageNet & $\mathbf{94.8 \pm 1.3}$ & $95.3 \pm 1.2$ & $95.3 \pm 1.3$ & $\mathbf{8.8 \pm 1.2}$ & \cellrel{$9.5 \pm 1.3$}{-7.2} & \cellrel{$10.9 \pm 2.1$}{-19.7} \\
  & ImageNet-A & $\mathbf{94.8 \pm 1.3}$ & $95.0 \pm 1.4$ & $95.4 \pm 1.2$ & $\mathbf{1.9 \pm 0.2}$ & \cellrel{$1.9 \pm 0.3$}{-3.1} & \cellrel{$2.1 \pm 0.3$}{-12.2} \\
  & ImageNet-R & $\mathbf{94.8 \pm 1.5}$ & $95.4 \pm 1.1$ & $95.7 \pm 1.0$ & $\mathbf{1.0 \pm 0.0}$ & \cellrel{$1.2 \pm 0.1$}{-18.2} & \cellrel{$1.2 \pm 0.1$}{-15.1} \\
  & EuroSAT & $\mathbf{94.4 \pm 1.6}$ & $95.4 \pm 1.3$ & $95.3 \pm 1.2$ & $\mathbf{4.5 \pm 0.3}$ & \cellrel{$4.6 \pm 0.2$}{-2.2} & \cellrel{$5.4 \pm 0.2$}{-17.8} \\
\midrule
 \multirow{6}{*}{MobileCLIP2-S2} & CIFAR-10 & $\mathbf{95.1 \pm 2.2}$ & $97.8 \pm 0.7$ & $97.8 \pm 0.6$ & $\mathbf{1.0 \pm 0.0}$ & \cellrel{$1.0 \pm 0.0$}{-4.0} & \cellrel{$1.0 \pm 0.0$}{-4.0} \\
  & CIFAR-100 & $\mathbf{95.4 \pm 1.9}$ & $95.7 \pm 1.8$ & $95.6 \pm 1.9$ & $\mathbf{1.6 \pm 0.1}$ & \cellrel{$1.6 \pm 0.1$}{-0.9} & \cellrel{$1.6 \pm 0.2$}{-1.5} \\
  & ImageNet & $\mathbf{95.0 \pm 1.2}$ & $95.1 \pm 1.4$ & $95.3 \pm 1.4$ & $\mathbf{4.0 \pm 0.5}$ & \cellrel{$4.4 \pm 0.6$}{-8.2} & \cellrel{$4.4 \pm 0.7$}{-8.0} \\
  & ImageNet-A & $\mathbf{95.2 \pm 1.2}$ & $95.3 \pm 1.3$ & $95.0 \pm 1.2$ & $\mathbf{20.8 \pm 2.1}$ & \cellrel{$22.1 \pm 2.3$}{-6.2} & \cellrel{$22.0 \pm 2.5$}{-5.7} \\
  & ImageNet-R & $\mathbf{95.2 \pm 1.4}$ & $95.3 \pm 1.4$ & $95.0 \pm 1.4$ & $\mathbf{1.8 \pm 0.2}$ & \cellrel{$2.0 \pm 0.3$}{-6.4} & \cellrel{$1.9 \pm 0.2$}{-0.6} \\
  & EuroSAT & $\mathbf{95.1 \pm 1.3}$ & $95.2 \pm 1.4$ & $95.3 \pm 1.2$ & $\mathbf{4.6 \pm 0.2}$ & \cellrel{$4.8 \pm 0.2$}{-3.4} & \cellrel{$5.7 \pm 0.2$}{-19.1} \\
\midrule
 \multirow{6}{*}{MobileCLIP2-S3} & CIFAR-10 & $\mathbf{95.7 \pm 1.8}$ & $97.8 \pm 0.7$ & $97.6 \pm 0.7$ & $\mathbf{1.0 \pm 0.0}$ & \cellrel{$1.0 \pm 0.0$}{-3.4} & \cellrel{$1.0 \pm 0.0$}{-3.4} \\
  & CIFAR-100 & $\mathbf{95.4 \pm 1.8}$ & $95.5 \pm 1.8$ & $95.7 \pm 1.8$ & $\mathbf{1.5 \pm 0.2}$ & \cellrel{$1.6 \pm 0.2$}{-3.4} & \cellrel{$1.6 \pm 0.2$}{-3.4} \\
  & ImageNet & $\mathbf{95.1 \pm 1.2}$ & $95.4 \pm 1.4$ & $95.2 \pm 1.3$ & $\mathbf{3.1 \pm 0.4}$ & \cellrel{$3.2 \pm 0.4$}{-1.7} & \cellrel{$3.3 \pm 0.4$}{-6.1} \\
  & ImageNet-A & $\mathbf{95.0 \pm 1.4}$ & $95.1 \pm 1.3$ & $95.0 \pm 1.2$ & $\mathbf{12.6 \pm 1.3}$ & \cellrel{$13.0 \pm 1.2$}{-3.6} & \cellrel{$13.8 \pm 1.1$}{-9.1} \\
  & ImageNet-R & $\mathbf{95.3 \pm 1.3}$ & $95.2 \pm 1.2$ & $95.2 \pm 1.2$ & $\mathbf{1.3 \pm 0.1}$ & \cellrel{$1.4 \pm 0.1$}{-1.3} & \cellrel{$1.4 \pm 0.1$}{-1.4} \\
  & EuroSAT & $\mathbf{95.0 \pm 1.2}$ & $95.1 \pm 1.3$ & $95.2 \pm 1.1$ & $\mathbf{4.3 \pm 0.2}$ & \cellrel{$4.9 \pm 0.2$}{-12.0} & \cellrel{$5.3 \pm 0.2$}{-18.5} \\
\midrule
 \multirow{6}{*}{MobileCLIP2-S4} & CIFAR-10 & $\mathbf{95.5 \pm 1.9}$ & $97.5 \pm 0.7$ & $97.7 \pm 0.7$ & $\mathbf{1.0 \pm 0.0}$ & \cellrel{$1.0 \pm 0.0$}{-3.9} & \cellrel{$1.0 \pm 0.0$}{-3.9} \\
  & CIFAR-100 & $\mathbf{95.5 \pm 2.1}$ & $95.6 \pm 1.9$ & $95.5 \pm 1.7$ & $\mathbf{1.3 \pm 0.1}$ & \cellrel{$1.3 \pm 0.1$}{-1.2} & \cellrel{$1.3 \pm 0.1$}{-0.3} \\
  & ImageNet & $\mathbf{95.1 \pm 1.2}$ & $95.0 \pm 1.4$ & $95.1 \pm 1.4$ & $\mathbf{3.0 \pm 0.4}$ & \cellrel{$3.0 \pm 0.4$}{-0.4} & \cellrel{$3.1 \pm 0.7$}{-4.0} \\
  & ImageNet-A & $\mathbf{94.9 \pm 1.3}$ & $95.2 \pm 1.3$ & $94.9 \pm 1.2$ & $\mathbf{8.5 \pm 0.9}$ & \cellrel{$9.3 \pm 1.3$}{-9.1} & \cellrel{$9.2 \pm 1.2$}{-7.5} \\
  & ImageNet-R & $\mathbf{95.1 \pm 1.2}$ & $95.3 \pm 1.2$ & $95.2 \pm 1.4$ & $\mathbf{1.2 \pm 0.1}$ & \cellrel{$1.2 \pm 0.1$}{-2.1} & \cellrel{$1.2 \pm 0.1$}{-1.7} \\
  & EuroSAT & $\mathbf{95.0 \pm 1.5}$ & $95.2 \pm 1.2$ & $95.1 \pm 1.5$ & $\mathbf{4.1 \pm 0.2}$ & \cellrel{$4.4 \pm 0.2$}{-7.0} & \cellrel{$4.5 \pm 0.3$}{-10.2} \\
\midrule
 \multirow{6}{*}{MobileCLIP2-B} & CIFAR-10 & $\mathbf{95.5 \pm 2.0}$ & $97.7 \pm 0.6$ & $97.1 \pm 0.7$ & $\mathbf{1.0 \pm 0.0}$ & \cellrel{$1.0 \pm 0.0$}{-3.9} & \cellrel{$1.0 \pm 0.0$}{-3.9} \\
  & CIFAR-100 & $\mathbf{95.5 \pm 1.8}$ & $95.7 \pm 2.0$ & $96.0 \pm 1.7$ & $\mathbf{1.5 \pm 0.1}$ & \cellrel{$1.5 \pm 0.1$}{-2.5} & \cellrel{$1.6 \pm 0.2$}{-4.1} \\
  & ImageNet & $\mathbf{95.3 \pm 1.1}$ & $95.3 \pm 1.2$ & $95.2 \pm 1.3$ & $\mathbf{3.5 \pm 0.7}$ & \cellrel{$3.5 \pm 0.5$}{-0.5} & \cellrel{$3.6 \pm 0.8$}{-4.2} \\
  & ImageNet-A & $\mathbf{95.1 \pm 1.4}$ & $95.1 \pm 1.5$ & $95.3 \pm 1.3$ & $\mathbf{16.6 \pm 2.4}$ & \cellrel{$17.1 \pm 1.9$}{-3.2} & \cellrel{$18.6 \pm 1.6$}{-10.6} \\
  & ImageNet-R & $\mathbf{94.9 \pm 1.4}$ & $95.1 \pm 1.5$ & $95.2 \pm 1.4$ & $\mathbf{1.5 \pm 0.2}$ & \cellrel{$1.5 \pm 0.2$}{-6.3} & \cellrel{$1.5 \pm 0.1$}{-1.9} \\
  & EuroSAT & $\mathbf{94.9 \pm 1.4}$ & $95.0 \pm 1.3$ & $95.5 \pm 1.4$ & $\mathbf{4.0 \pm 0.2}$ & \cellrel{$4.5 \pm 0.2$}{-10.4} & \cellrel{$4.9 \pm 0.2$}{-18.3} \\
\end{longtable}
\endgroup

%% file: AISTATS/sections/Appendix/Supp_LLM.tex
\textbf{Motivation.}
Applying conformal prediction to open-ended LLM responses is more difficult than applying it to classification or multiple-choice tasks. The output space is large and the correctness of a response is not always binary. A common practical approach is to use a stronger model as an evaluator and assign each response a quality score. However, these evaluator scores are often coarse. For example, if the judge model scores answers on a scale from 0 to 10 with only one decimal place, many calibration examples can receive the same score. This creates many ties in the conformal ranking and makes it difficult to obtain a fine ordering of responses near the calibration threshold.

The $r$-value is useful in this setting because it does not rely only on the average evaluator score. Instead, it also uses the variability of the score across different prompt formulations or evaluation perspectives. Thus, two responses with the same average score can still receive different $r$-values if one response is evaluated more stably than the other. This does not remove the fundamental difficulty of open-ended conformal prediction, but it provides a more informative ranking when the original score is coarse or tied.

\textbf{Dataset preparation and question diversification.}
We use TruthfulQA~\cite{truthfulqa} to study this issue. For each original question $Q$, we generate $p$ semantically equivalent rephrased questions $\{q_1,\ldots,q_p\}$ using GPT-4o. To ensure that the rephrased questions preserve the original meaning, we compute the sentence similarity between the original question and each rephrased version using the all-MiniLM-L6-v2 sentence embedding model. We keep only rephrased questions whose cosine similarity with the original question is above a threshold $\tau$. Each original question and its retained rephrases are associated with the validated ground truth answers provided by TruthfulQA.

\textbf{Response generation and selection.}
For each question variant, we use LLaMA 3B with temperature sampling to generate multiple candidate responses. We then compare each generated response with the ground truth answers using embedding-based cosine similarity. The selected response $A^*$ is the one with the largest average similarity to the ground truth answers,
\[
    A^*
    =
    \arg\max_A
    \frac{1}{|GT|}
    \sum_{g\in GT}
    \mathrm{CosSim}\bigl(\mathrm{embed}(A), \mathrm{embed}(g)\bigr),
\]
where $GT$ denotes the set of ground truth answers. This step gives a representative response for each question while still allowing the generation process to capture variation across sampled answers.

\textbf{Multi-perspective evaluation.}
We then evaluate each selected answer using GPT-4o as a judge model. For a selected answer $A^*$, we present the original question and its rephrased versions to the judge. This produces a collection of evaluation scores
\[
    \{S_0,S_1,\ldots,S_p\},
\]
where each score is on a scale from 0 to 10 and reflects the judged quality of the same answer under a slightly different formulation of the question. These scores provide a simple way to estimate both the average quality of the answer and the uncertainty of the evaluator score.

\textbf{Computing the $r$-value.}
For each response, we compute the mean evaluator score
\[
    \mu = \frac{1}{p+1}\sum_{i=0}^{p} S_i
\]
and the empirical score variance across rephrased evaluations. The $r$-value is then computed using the empirical Bayes formulation described in the main text. In this experiment, the role of variability is to distinguish responses that have the same or nearly the same average score but different levels of evaluation stability. Responses with high and stable scores receive better rankings, while responses whose high scores are less stable are penalized.

\textbf{$r$-value helps with tied calibration scores.}
Table~\ref{tab:integer_rank} shows an example from the TruthfulQA experiment. Many responses receive identical or nearly identical evaluator scores. For example, several responses have score $10.0$ and therefore share the same score rank. Standard score-based conformal prediction cannot distinguish these responses further. In contrast, the $r$-value uses the estimated score variability to refine the ranking. Responses with the same score can receive different $r$-values because their evaluation stability differs.

This result should be interpreted as a ranking improvement rather than a complete solution to open-ended conformal prediction. The conformal guarantee still depends on the calibration setup and exchangeability, and open-ended response evaluation remains inherently difficult. Nevertheless, the experiment shows that $r$-values can provide useful additional resolution when evaluator scores are discrete, coarse, or heavily tied.

\begin{table}[t]
    \centering
    \caption{$r$-value refines the ranking when evaluator scores are tied or nearly tied. Smaller $r$-values indicate more stable high-quality responses.}
    \begin{tabular}{ccccc}
        \toprule
        \bf{$r$-value} & \bf{$r$-value Rank} & \bf{Score Rank} & \bf{Score} & \bf{SE} \\
        \midrule
        0.03357091 & 1 & 6.0 & 10.0 & 0.712 \\
        0.03542047 & 2 & 6.0 & 10.0 & 0.557 \\
        0.03590146 & 3 & 6.0 & 10.0 & 0.530 \\
        0.03680355 & 4 & 6.0 & 10.0 & 1.340 \\
        0.03691267 & 5 & 21.0 & 9.9 & 0.902 \\
        0.03697749 & 6 & 21.0 & 9.9 & 0.840 \\
        0.03728661 & 7 & 21.0 & 9.9 & 0.768 \\
        0.03811422 & 8 & 21.0 & 9.9 & 0.673 \\
        0.03829784 & 9 & 21.0 & 9.9 & 0.657 \\
        0.03933026 & 10 & 21.0 & 9.9 & 1.255 \\
        \bottomrule
    \end{tabular}
    \label{tab:integer_rank}
\end{table}

\begin{figure}[h]
    \centering
    \includegraphics[width=0.95\linewidth]{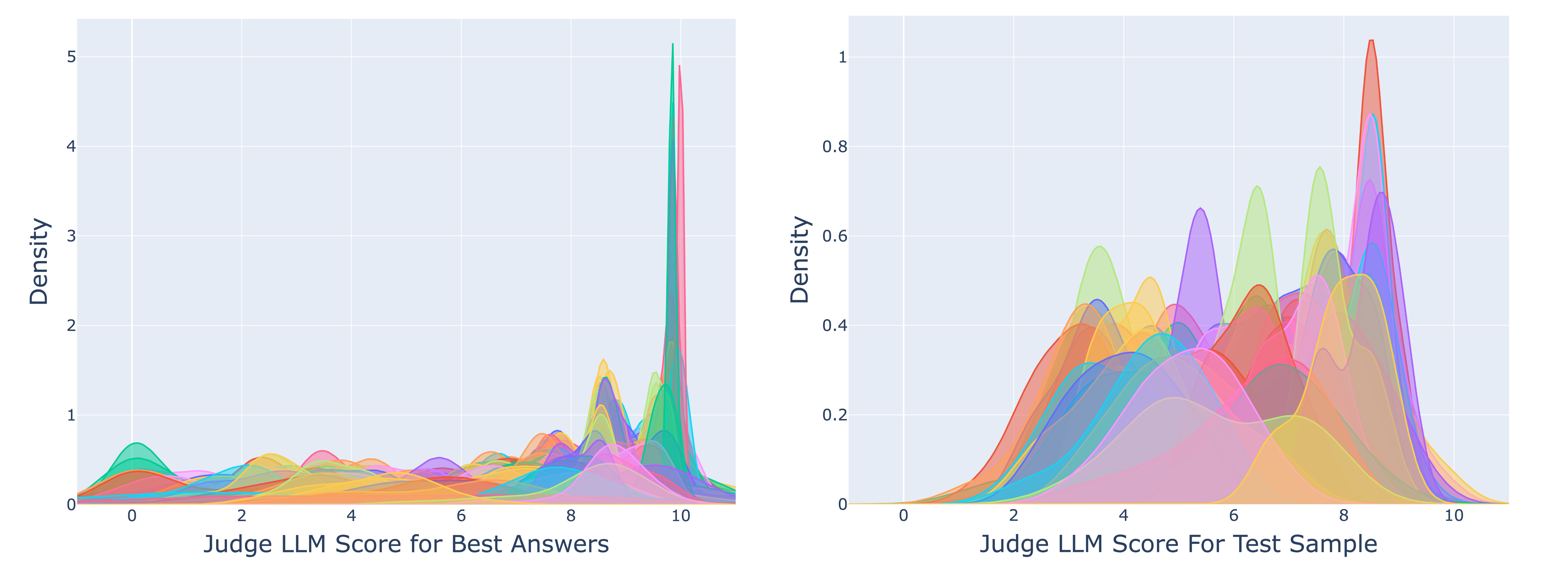}
    \caption{Smoothed distribution of judge scores for the selected answer $A^*$ and responses generated under temperature sampling for a TruthfulQA prompt. The score distribution illustrates that evaluator scores can vary across response samples and prompt formulations.}
    \label{fig:score_dist}
\end{figure}

%% file: AISTATS/sections/Appendix/Supp_LLM_Close.tex
\paragraph{Comparison with Conformal Language Model}

CLM significantly advances open-ended text generation by clustering sampled candidate responses. However, our LLM experiments focus on tasks like MMLU, which have a finite/discrete set of candidate answers. This setting meaningfully differs from open-ended generation. We conducted experiments adapting CLM to MMLU and confirmed the difficulties; CLM’s sampling-based procedure struggles with a small fixed set of responses, leading to an increased risk of returning a null set (Algorithm 1 in CLM paper). By setting the similarity between options to 0 (to reduce the risk), we find that CLM does not produce valid prediction sets. We acknowledge that methods for open-ended tasks could be extendable to finite sets but would require significant modifications.

\begin{table}[h]
\centering
\caption{Average conformal set size comparison across models.}
\label{tab:llm_set_size}
\begin{tabular}{lccc}
\toprule
Model & $\CPr$ & $\CPavg$ & $\CP$ \\
\midrule
LLaMA 3B & $\mathbf{2.39}$ & $2.55\;(-6.43\%)$ & $2.56\;(-7.15\%)$ \\
Mistral  & $\mathbf{2.30}$ & $2.42\;(-5.49\%)$ & $2.47\;(-7.49\%)$ \\
Phi 3.5  & $\mathbf{2.17}$ & $2.33\;(-7.48\%)$ & $2.24\;(-3.51\%)$ \\
Qwen 7B  & $\mathbf{1.94}$ & $2.08\;(-6.85\%)$ & $2.07\;(-6.44\%)$ \\
\bottomrule
\end{tabular}
\end{table}

\begin{table}[h]
  \centering
  \caption{Set sizes for variousg subjects of MMLU at $\alpha=0.05$.}
  \label{tab:set-sizes}
  \begin{tabular}{lrr}
    \toprule
    Subject                         & \multicolumn{1}{c}{$\CPr$} & \multicolumn{1}{c}{$\CPavg$} \\
    \midrule
    computer security               & \textbf{3.675}                        & 3.684                           \\
    high school computer science    & \textbf{3.802}                        & 3.806                           \\
    college computer science        & \textbf{3.886}                       & 3.888                           \\
    machine learning                & 3.825                        & \textbf{3.814}                           \\
    formal logic                    & \textbf{3.846}                        & 3.868                           \\
    high school biology             & \textbf{3.490}                        & 3.492                           \\
    anatomy                         & \textbf{3.307}                        & 3.340                           \\
    clinical knowledge              & \textbf{3.307}                        & 3.422                           \\
    college medicine                & \textbf{3.593}                        & 3.620                           \\
    professional medicine           & \textbf{3.613}                        & 3.615                           \\
    college chemistry               &\textbf{3.772}                        & 3.813                           \\
    marketing                       & \textbf{2.754}                        & 2.865                           \\
    public relations                & \textbf{3.546}                        & 3.552                           \\
    management                      & \textbf{3.362}                        & 3.381                           \\
    business ethics                 & 3.776                        & \textbf{3.771}                           \\
    professional accounting         & 3.766                        & \textbf{3.758}                           \\
    \bottomrule
  \end{tabular}
\end{table}

%% file: AISTATS/sections/Appendix/prompts.tex
\begin{figure*}[h]
    \begin{AIbox}{VLM Rephrasing Prompt}
    \scriptsize
    \begin{minipage}{0.95\linewidth}
    \vspace{3pt}
    \raggedright

    \noindent\textbf{System prompt.}

    \vspace{3pt}

    You are a creative AI tasked with generating rephrasings of short image descriptions.
    Your task is to:
    \begin{enumerate}
        \item Read the description carefully.
        \item Generate \texttt{\{num\_rephrases\}} diverse and natural-sounding rephrasings of the description.
        \item Keep the core meaning the same.
        \item Return a JSON object with the key \texttt{rephrased\_descriptions}, containing a list of strings.
    \end{enumerate}

    \noindent\textbf{User prompt.}

    \vspace{3pt}

    Original description: \texttt{\{input\_image\_description\}}

    \vspace{3pt}
    \end{minipage}
    \end{AIbox}
    \label{fig:vlm_rephrasing_prompt}
\end{figure*}

\begin{figure*}[h]
    \begin{AIbox}{GPQA Rephrasing Prompt}
    \scriptsize
    \begin{minipage}{0.95\linewidth}
    \vspace{3pt}
    \raggedright

    \noindent\textbf{System prompt.}

    \vspace{3pt}

    You are an expert exam question rephraser. Your task is to:
    \begin{enumerate}
        \item Read the question and its answer options carefully.
        \item Generate different rephrased versions of the question while maintaining relevance to the answer options.
        \item Return the rephrased versions as a JSON object with the key \texttt{rephrased\_questions}, containing a list of strings.
        \item Return exactly the requested number of rephrased questions.
    \end{enumerate}

    \vspace{3pt}

    \noindent\textbf{User prompt.}

    \vspace{3pt}

    Rephrase the following question while considering the answer options.

    \vspace{3pt}

    \noindent\textbf{Question:} \texttt{\{input\_text\}}

    \vspace{3pt}

    \noindent\textbf{Answer options:} \texttt{\{options\_text\}}

    \vspace{3pt}

    \noindent\textbf{Number of rephrased questions to return:} \texttt{\{num\_rephrases\}}

    \vspace{3pt}
    \end{minipage}
    \end{AIbox}
    \label{fig:gpqa_rephrasing_prompt}
\end{figure*}

\begin{figure*}[h]
    \begin{AIbox}{TruthfulQA Evaluation Prompt}
    \scriptsize
    \begin{minipage}{0.95\linewidth}
    \vspace{3pt}
    \raggedright

    \noindent\textbf{System prompt.}

    \vspace{3pt}

    You are an expert evaluator assessing how well answers address questions. Your task is to:
    \begin{enumerate}
        \item Read each question and the provided answer carefully.
        \item Evaluate how directly and completely the answer addresses each question.
        \item Consider factors such as accuracy, relevance, and comprehensiveness.
        \item Assign a score from 0 to 10, with one decimal place, for each question.
        \item Return the scores as a JSON object with a \texttt{scores} array.
        \item Use the full range of decimal scores, avoid round numbers, and avoid assigning the same score to every question.
    \end{enumerate}

    \vspace{3pt}

    \noindent\textbf{User prompt.}

    \vspace{3pt}

    Evaluate how well the following answer addresses each question. Return only a JSON object with a \texttt{scores} array containing scores from 0 to 10 with one decimal place.

    \vspace{3pt}

    \noindent\textbf{Answer to evaluate:} \texttt{\{answer\}}

    \vspace{3pt}

    \noindent\textbf{Questions:} \texttt{\{formatted\_questions\}}

    \vspace{3pt}
    \end{minipage}
    \end{AIbox}
    \label{fig:truthfulqa_evaluation_prompt}
\end{figure*}

%% file: AISTATS/sections/Appendix/Impact.tex
\textbf{Positive Impact} Our proposed $r$-value conformal prediction method provides ordered prediction sets with a formal coverage guarantee, which ranks the most likely labels first and can potentially contribute in high-stakes domains. 
For example, in healthcare, clinicians can focus on the top-ranked diagnoses, which can save critical time and reduce diagnostic error; in finance, investigators can classify transactions by predicted fraud risk, focusing efforts on the highest-risk cases; and in legal settings, risk assessment tools can list factors in order of importance, making decisions clearer and easier to explain. 

\textbf{Negative Impact} If one relies only on the hightest-ranked labels and ignore the rest, they might miss other valid options further down the list. In addition, the performance of our method depends on the quality of the data and also how accurate the variability estimation is. If either is poor, then the statistical guarantee may be violated, or the ranking may be misleading and make the result not trustful.